%% file: arxiv.tex
\documentclass{article}

    \PassOptionsToPackage{numbers, compress}{natbib}
    \usepackage[preprint]{neurips_2020}

\usepackage[utf8]{inputenc} %
\usepackage[T1]{fontenc}    %
\usepackage{hyperref}       %
\usepackage{url}            %
\usepackage{booktabs}       %
\usepackage{amsfonts}       %
\usepackage{nicefrac}       %
\usepackage{microtype}      %

\usepackage{graphicx}
\usepackage{subfig}
\usepackage{multirow}
\usepackage{color}
\input{math_commands.tex}

\usepackage{algorithm}
\usepackage[noend]{algorithmic}

\usepackage{amsthm}

\newtheorem{proposition}{Proposition}
\theoremstyle{remark}

\theoremstyle{definition}

\usepackage{hyperref}

\newcommand{\bhline}[1]{\noalign{\hrule height #1}}

\title{Deployment-Efficient Reinforcement Learning via Model-Based Offline Optimization}

\author{
  Tatsuya Matsushima\thanks{Equal contribution.} \qquad Hiroki Furuta\footnotemark[1]\qquad\quad\\
  The University of Tokyo\\
  \texttt{\{matsushima, furuta\}@weblab.t.u-tokyo.ac.jp} \\
  \AND
  Yutaka Matsuo \\
  The University of Tokyo\\
  \texttt{matsuo@weblab.t.u-tokyo.ac.jp} \\
  \And
  Ofir Nachum \\
  Google Research \\
  \texttt{ofirnachum@google.com}
  \And
  Shixiang Shane Gu \\
  Google Research \\
  \texttt{shanegu@google.com} \\
}

\begin{document}

\maketitle

\begin{abstract}
    Most reinforcement learning~(RL) algorithms assume online access to the environment, in which one may readily interleave updates to the policy with experience collection using that policy. However, in many real-world applications such as health, education, dialogue agents, and robotics, the cost or potential risk of deploying a new data-collection policy is high, to the point that it can become prohibitive to update the data-collection policy more than a few times during learning. With this view, we propose a novel concept of \emph{deployment efficiency}, measuring the number of distinct data-collection policies that are used during policy learning. We observe that na\"{i}vely applying existing model-free offline RL algorithms recursively does not lead to a practical deployment-efficient \textit{and} sample-efficient algorithm. We propose a novel model-based algorithm, Behavior-Regularized Model-ENsemble~(BREMEN) that can effectively optimize a policy offline using 10-20 times fewer data than prior works. Furthermore, the recursive application of BREMEN is able to achieve impressive deployment efficiency while maintaining the same or better sample efficiency, learning successful policies from scratch on simulated robotic environments with only 5-10 deployments, compared to typical values of hundreds to millions in standard RL baselines. Codes and pre-trained models are available at~\url{https://github.com/matsuolab/BREMEN}.
\end{abstract}

\section{Introduction}

Reinforcement learning~(RL) algorithms have recently demonstrated impressive success in learning behaviors for a variety of sequential decision-making tasks~\cite{d4pg,hessel2018rainbow,nachum2019multi}.
Virtually all of these demonstrations have relied on highly-frequent online access to the environment, with the RL algorithms often interleaving each update to the policy with additional experience collection of that policy acting in the environment.
However, in many real-world applications of RL, such as health~\cite{murphy2001marginal}, education~\cite{mandel2014offline}, dialog agents~\citep{jaques2019way}, and robotics~\cite{gu2017deep,kalashnikov2018qtopt}, the deployment of a new data-collection policy may be associated with a number of costs and risks.
If we can learn tasks with a small number of data collection policies, we can substantially reduce these costs and risks.

Based on this idea, we propose a novel measure of RL algorithm performance, namely \emph{deployment efficiency}, which counts the number of changes in the data-collection policy during learning, as illustrated in Figure~\ref{fig:eyecatcher}. %
This concept may be seen in contrast to \emph{sample efficiency} or \emph{data efficiency}~\cite{precup2001off,degris2012off,gu2016qprop,Haarnoja:2018uya,Lillicrap2016,nachum2018data}, 
which measures the amount of environment interactions incurred during training, without regard to how many distinct policies were deployed to perform those interactions.
Even when the data efficiency is high, the deployment efficiency could be low, since many on-policy and off-policy algorithms alternate data collection with each policy update~\citep{Schulman:2015uk,Lillicrap2016,gu2016continuous,Haarnoja:2018uya}.
Such dependence on high-frequency policy deployments is best illustrated in the recent works in offline RL~\cite{Fujimoto:2018td,jaques2019way,kumar2019stabilizing,levine2020offline,Wu:2019wl}, where baseline off-policy algorithms exhibited poor performance when trained on a static dataset.
These offline RL works, however, limit their study to a single deployment, which is enough for achieving high performance with data collected from a sub-optimal behavior policy, but often not from a random policy.
In contrast to those prior works, we aim to learn successful policies from scratch with minimal amounts of data and deployments.
Many existing model-free offline RL algorithms~\citep{levine2020offline} are tuned and evaluated on large datasets (e.g., one million transitions). 
In order to develop an algorithm that is both sample-efficient and deployment-efficient, each iteration of the algorithm between successive deployments has to work effectively on much smaller dataset sizes.
We believe model-based RL is better suited to this setting due to its higher demonstrated sample efficiency than model-free RL~\cite{Kurutach:2018wq, Nagabandi2018}. 
Although the combination of model-based RL and offline or limited-deployment settings seems straight-forward, we find this na\"{i}ve approach leads to poor performance.
This problem can be attributed to \emph{extrapolation errors}~\cite{Fujimoto:2018td} similar to those observed in model-free methods. Specifically, the learned policy may choose sequences of actions which lead it to regions of the state space where the dynamics model cannot predict properly, due to poor coverage of the dataset. This can lead the policy to exploit approximation errors of the dynamics model and be disastrous for learning.
In model-free settings, similar data distribution shift problems are typically remedied by regularizing policy updates explicitly with a divergence from the observed data distribution~\cite{jaques2019way,kumar2019stabilizing,Wu:2019wl}, which, however, can overly limit policies' expressivity~\cite{sohn2020brpo}.

In order to better approach these problems arising in limited deployment settings, we propose Behavior-Regularized Model-ENsemble (BREMEN), which learns an ensemble of dynamics models in conjunction with a policy using imaginary rollouts while \emph{implicitly} regularizing the learned policy via appropriate parameter initialization and conservative trust-region learning updates.
We evaluate BREMEN on high-dimensional continuous control benchmarks and find that it achieves impressive deployment efficiency. BREMEN is able to learn successful policies with only 5-10 deployments, significantly outperforming existing off-policy and offline RL algorithms in this deployment-constrained setting.
We further evaluate BREMEN on standard offline RL benchmarks, where only a single static dataset is used. In this fixed-batch setting, our experiments show that BREMEN can not only achieve performance competitive with state-of-the-art when using standard dataset sizes but also learn with 10-20 times smaller datasets, which previous methods are unable to attain.

\begin{figure}[t]
    \centering
    \includegraphics[width=\linewidth]{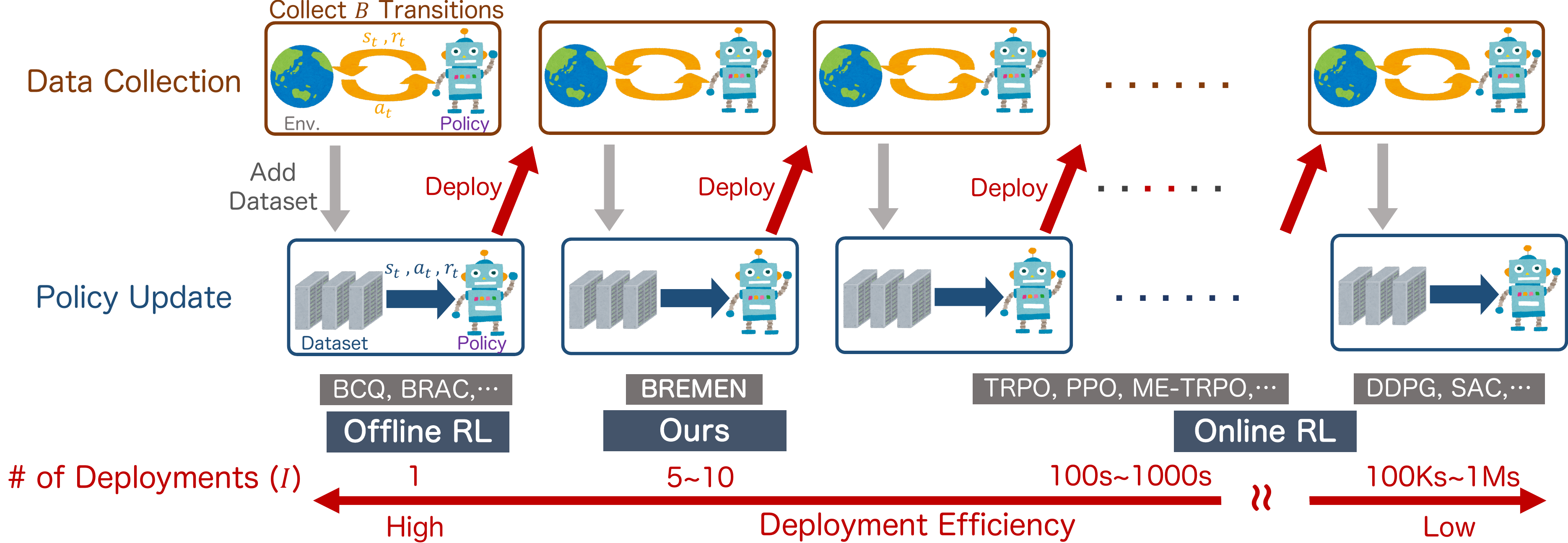}\label{fig:eyecatcher:deployment}
    \caption{
        \textit{Deployment efficiency} is defined as the number of changes in the data-collection policy~($I$), which is vital for managing costs and risks of new policy deployment. 
        Online RL algorithms typically require many iterations of policy deployment and data collection, which leads to extremely low deployment efficiency. In contrast, most pure offline algorithms consider updating a policy from a fixed dataset without additional deployment and often fail to learn from a randomly initialized data-collection policy. Interestingly, most state-of-the-art off-policy algorithms are still evaluated in heavily online settings. For example, SAC~\cite{Haarnoja:2018uya} collects one sample per policy update, amounting to 100,000 to 1 million deployments for learning standard benchmark domains.
        \vspace{-0.43cm}
    }
    \label{fig:eyecatcher}
    
\end{figure}

\section{Preliminaries}
We consider a Markov Decision Process~(MDP) setting, characterized by the tuple~$\mathcal{M}=(\mathcal{S},\mathcal{A},p,r,\gamma)$, where~$\mathcal{S}$ is the state space, $\mathcal{A}$ is the action space,~$p(s'|s,a)$ is the transition probability distribution or dynamics,~$r(s)$ is the reward function and~$\gamma \in (0,1)$ is the discount factor.
A policy~$\pi$ is a function that determines the agent behavior, mapping from states to probability distributions over actions.
The goal is to obtain the optimal policy~$\pi^{\ast}$ as
\begin{equation}
    \pi^{\ast} = \argmax_{\pi} \eta[\pi] = \argmax_{\pi} \mathbb{E}_{\pi} \left[\sum_{t=0}^{\infty}\gamma^t r(s_t)\right], \nonumber
    \label{eq:preliminary}
\end{equation}
where $\eta[\pi]$ is the expectation of the discounted sum of rewards under the policy~$\pi$.
The transition probability~$p(s'|s,a)$ is usually unknown, and it is estimated with a parameterized dynamics model~$f_{\phi}$ (e.g., a neural network) in model-based RL.
For simplicity, we assume that the reward function~$r(s)$ is known, and the reward can be computed for any arbitrary state, but we can easily extend to the unknown setting and predict it using a parameterized function.

\textbf{On-policy vs Off-policy, Online vs Offline} At high-level, most RL algorithms iterate many times between collecting a batch of transitions (deployments) and optimizing the policy (learning). If the algorithms discard data after each policy update, they are \textit{on-policy}~\citep{Schulman:2015uk,Schulman2017PPO}, while if they accumulate data in a buffer $\mathcal{D}$, i.e. experience replay~\citep{lin1992self}, they are \textit{off-policy}~\citep{mnih2015human,Lillicrap2016,gu2016continuous,gu2016qprop,Haarnoja:2018uya,Fujimoto:2018td} because not all the data in buffer comes from the current policy.
However, we consider all these algorithms to be \textit{online} RL algorithms, since they involve many deployments during learning, ranging from hundreds to millions.
On the other hand, in pure \textit{offline} RL, one does not assume direct interaction and learns a policy from only a fixed dataset, which effectively corresponds to a single deployment allowed for learning.
Classically, interpolating these two extremes were semi-batch RL algorithms~\cite{lange2012batch,Satinder1994}, which improve the policy through repetitions of collecting a large batch of transitions~$\mathcal{D}=\{(s,a,s',r)\}$ and performing many or full policy updates. While these semi-batch RL also realize good deployment efficiency, they have not been extensively studied with neural network function approximators or in off-policy settings with experience replay for scalable sample-efficient learning.
In our work, we aim to have both high deployment efficiency and sample efficiency by developing an algorithm that can solve the tasks with minimal policy deployments as well as transition samples.

\section{Deployment Efficiency}
\label{sec:def-deploymentefficiency}
Deploying a new policy for data collection can be associated with a number of costs and risks for many real-world applications like medicine or robotic control~\cite{murphy2001marginal,mandel2014offline,gu2017deep,kalashnikov2018qtopt,nachum2019multi}. While there is an abundance of works on safety for RL~\citep{chow2015risk,eysenbach2017leave,chow2018lyapunov,raybenchmarking,chow2019lyapunov}, these methods often do not provide guarantees in practice when combined with neural networks and stochastic optimization. It is therefore necessary to validate each policy before deployment. Due to the cost associated with each deployment, it is desirable to minimize the number of distinct deployments needed during the learning process.

In order to focus research on these practical bottlenecks, we propose a novel measure of RL algorithms, namely, \emph{deployment efficiency}, which counts how many times the data-collection policy has been changed during improvement from random policy to solve the task.
For example, if an RL algorithm operates by using its learned policy to collect transitions from the environment $I$ times, each time collecting a batch of $B$ new transitions, then the number of deployments is $I$, while the total number of samples collected is $I\times B$. The lower $I$ is, the more deployment-efficient the algorithm is; in contrast, sample efficiency looks at $I\times B$.
Online RL algorithms, whether they are on-policy or off-policy, typically update the policy and acquire new transitions by deploying the newly updated policy at every iteration.
This corresponds to performing hundreds to millions of deployments during learning on standard benchmarks~\citep{Haarnoja:2018uya}, which is severely deployment inefficient.
On the other hand, offline RL literature only studies the case of 1 deployment.
A deployment-efficient algorithm would stand in the middle of these two extremes and ideally learn a successful policy from scratch while deploying only a few distinct policies, as illustrated in Figure~\ref{fig:eyecatcher}.

Recent deep RL literature seldom emphasizes deployment efficiency, with few exceptions in specific applications~\cite{kalashnikov2018qtopt} where such a learning procedure is necessary. 
Although current state-of-the-art algorithms on continuous control have substantially improved sample or data efficiency, they have not optimized for deployment efficiency.
For example, SAC~\cite{Haarnoja:2018uya}, an efficient model-free off-policy algorithm, performs half a million to one million policy deployments during learning on MuJoCo~\cite{todorov2012mujoco} benchmarks. ME-TRPO~\cite{Kurutach:2018wq}, a model-based algorithm, performs a much lower 100-300 policy deployments, although this is still relatively high for practical settings.\footnote{We examined the number of deployments by checking their original implementations, while the frequency of data collection is a tunable hyper-parameter.}
In our work, we demonstrate successful learning on standard benchmark environments with only 5-10 deployments.

\section{Behavior-Regularized Model-Ensemble}
To achieve high deployment efficiency, we propose Behavior-Regularized Model-ENsemble~(BREMEN).
BREMEN incorporates Dyna-style~\cite{sutton1991dyna} model-based RL, learning an ensemble of dynamics models in conjunction with a policy using imaginary rollouts from the ensemble and behavior regularization via conservative trust-region updates.

\subsection{Imaginary Rollout from Model Ensemble}
As in recent Dyna-style model-based RL methods~\cite{Kurutach:2018wq,Wang:2019vw}, BREMEN uses an ensemble of $K$ deterministic dynamics models~$\hat{f}_{\phi}=\left\{\hat{f}_{\phi_{1}}, \dots, \hat{f}_{\phi_{K}}\right\}$ to alleviate the problem of model bias.
Each model~$\hat{f}_{\phi_{i}}$ is parameterized by~${\phi_{i}}$ and trained by the following objective, which minimizes mean squared error between the prediction of next state~$\hat{f}_{\phi_{i}}(s_t,a_t)$ and true next state~$s_{t+1}$ over a dataset~$\mathcal{D}$:
\begin{equation}
    \label{eq:model-obj}
    \min _{\phi_i} \frac{1}{|\mathcal{D}|} \sum_{\left(s_{t}, a_{t}, s_{t+1}\right) \in \mathcal{D}}\frac{1}{2}\left\|s_{t+1}-\hat{f}_{\phi_i}\left(s_{t}, a_{t}\right)\right\|_{2}^{2}.
\end{equation}
During training of a policy $\pi_\theta$, imagined trajectories of states and actions are generated sequentially, using a dynamics model~$\hat{f}_{\phi_{i}}$ that is randomly selected at each time step: 
\begin{equation}
   \label{eq:rollout}
    a_{t} \sim \pi_{\theta}(\cdot|\hat{s}_{t}), \quad \hat{s}_{t+1} = \hat{f}_{\phi_{i}}(\hat{s}_{t},a_t) \quad \text{where} \quad i\sim \{1\cdots K\}.
\end{equation}

\subsection{Policy Update with Behavior Regularization}
In order to manage the discrepancy between the true dynamics and the learned model caused by the distribution shift in batch settings, we propose to use iterative policy updates via a trust-region constraint, re-initialized with a behavior-cloned policy after every deployment.
Specifically, after each deployment, we are given an updated dataset of experience transitions $\mathcal{D}$. With this dataset, we approximate the true behavior policy~$\pi_{b}$ through behavior cloning~(BC), utilizing a neural network $\hat{\pi}_{\beta}$ parameterized by~$\beta$, where we implicitly assume a fixed variance, a common practice in BC~\citep{rajeswaran2017learning}:
\begin{equation}
    \label{eq:bc-obj}
    \min _{\beta} \frac{1}{|\mathcal{D}|} \sum_{\left(s_{t}, a_{t}\right) \in \mathcal{D}}\frac{1}{2}\left\|a_{t}-\hat{\pi}_{\beta}\left(s_{t}\right)\right\|_{2}^{2}.
\end{equation}
After obtaining the estimated behavior policy, we initialize the target policy~$\pi_{\theta}$ as a Gaussian policy with mean from~$\hat{\pi}_{\beta}$ and standard deviation of $1$.
This BC initialization in conjunction with gradient descent based optimization may be seen as implicitly biasing the optimized $\pi_\theta$ to be close to the data-collection policy~\cite{nagarajan2019generalization}, and thus works as a remedy for the distribution shift problem~\citep{ross2011reduction}.
To further bias the learned policy to be close to the data-collection policy, we opt to use a KL-based trust-region optimization~\cite{Schulman:2015uk}. 
Therefore, the optimization of BREMEN becomes
\begin{align}
    \label{eq:trpo-obj}
    \theta_{k+1} &= \argmax_{\theta} \underset{s, a\sim \pi_{\theta_{k}}, \hat{f}_{\phi_i}}{\mathbb{E}}\left[\frac{\pi_{\theta}(a | s)}{\pi_{\theta_{k}}(a | s)} A^{\pi_{\theta_{k}}}(s, a)\right]\\ \nonumber
    &\text{s.t.} \quad \underset{s \sim \pi_{\theta_k}, \hat{f}_{\phi_i}}{\E}\left[\KL\left(\pi_{\theta}(\cdot | s) \| \pi_{\theta_{k}}(\cdot | s)\right)\right] \leq \delta,  \quad \pi_{\theta_0}=\mathrm{Normal}(\hat{\pi}_{\beta}, 1),%
\end{align}
where $A^{\pi_{\theta_{k}}}(s, a)$ is the advantage of $\pi_{\theta_k}$ computed using model-based rollouts in the learned dynamics model and $\delta$ is the maximum step size.

The combination of BC for initialization and finite iterative trust-region updates serves as an implicit KL regularization, as discussed in Section~\ref{sec:math_intuition}. This is in contrast to many previous offline RL algorithms that augment the value function with a penalty of explicit KL divergence~\cite{Siegel2020KeepDW,Wu:2019wl} or maximum mean discrepancy~\cite{kumar2019stabilizing}.
Empirically, we found that our regularization technique outperforms the explicit KL penalty (see Section~\ref{sec:implicit-kl}).
By recursively performing offline procedure, BREMEN can be used for deployment-efficient learning as shown in Algorithm~\ref{alg:bremen_deploy}, starting from a randomly initialized policy, collecting experience data, and performing offline policy updates.

\begin{algorithm}[t]
    \small
    \caption{BREMEN for Deployment-Efficient RL}
    \label{alg:bremen_deploy}    
    \renewcommand{\algorithmicrequire}{\textbf{Input:}}
    \begin{algorithmic}[1]          
        \REQUIRE Empty dataset $\mathcal{D}_{all}$, $\mathcal{D}$, Initial parameters $\phi = \{\phi_1,\cdots,\phi_K\}$, $\beta$, Number of policy optimization $T$, Number of deployments $I$.
        \STATE Randomly initialize the target policy $\pi_{\theta}$
        \FOR{deployment $i=1,\cdots,I$}
            \STATE Collect $B$ transitions in the true environment using $\pi_{\theta}$ and add them to dataset \\ $\mathcal{D}_{all} \leftarrow \mathcal{D}_{all} \cup \{s_t,a_t,r_t,s_{t+1}\}$, $\mathcal{D} \leftarrow \{s_t,a_t,r_t,s_{t+1}\}$
            \STATE Train $K$ dynamics models $\hat{f}_{\phi}$ using $\mathcal{D}_{all}$  via Eq.~\ref{eq:model-obj}.
            \STATE Train estimated behavior policy $\hat{\pi}_{\beta}$ using $\mathcal{D}$ by behavior cloning via Eq.~\ref{eq:bc-obj}.
            \STATE Re-initialize target policy $\pi_{\theta_0}=\mathrm{Normal}(\hat{\pi}_{\beta}, 1)$.
            \FOR{policy optimization $k=1,\cdots,T$}
                \STATE Generate imaginary rollout via Eq.~\ref{eq:rollout}.
                \STATE Optimize target policy $\pi_{\theta}$ satisfying Eq.~\ref{eq:trpo-obj} with the rollout.
            \ENDFOR
        \ENDFOR
    \end{algorithmic}
\end{algorithm}

\subsection{Implicit KL Control from a Mathematical Perspective}
\label{sec:math_intuition}
We can intuitively understand that behavior cloning initialization with trust-region updates works as a regularization of distributional shift, and this can be supported by theory.
Following the notation of~\citet{janner2019trust}, we denote the generalization error of a dynamics model on the state distribution under the true behavior policy as~$\epsilon_{m} = \max_{t} \mathbb{E}_{s\sim d^{\pi_{b}}_t} D_{TV} (p(s_{t+1}|s_t,a_t)||p_{\phi}(s_{t+1}|s_t,a_t))$, where $D_{TV}$ represents the total variation distance between true dynamics $p$ and learned model $p_{\phi}$.
We also denote the distribution shift on the target policy as~$\max_{s} D_{TV}(\pi_{b}||\pi) \leq \epsilon_{\pi}$.
A bound relating the true returns $\eta[\pi]$ and the model returns $\hat{\eta}[\pi]$ on the target policy is given in~\citet{janner2019trust} as,
\begin{equation}
    \label{eq:mbpo_bound}
    \eta[\pi] \geq \hat{\eta}[\pi] - \left[\frac{2\gamma r_{max}(\epsilon_{m}+2\epsilon_{\pi})}{(1-\gamma)^2} + \frac{4r_{max}\epsilon_{\pi}}{(1-\gamma)}\right].
\end{equation}
This bound guarantees the improvement under the true returns as long as the improvement under the model returns increases by more than the slack in the bound due to $\epsilon_m,\epsilon_\pi$~\cite{janner2019trust,levine2020offline}.

We may relate this bound to the specific learning employed by BREMEN, which includes dynamics model learning, behavior cloning policy initialization, and conservative KL-based trust-region policy updates. To do so, we consider an \emph{idealized} version of BREMEN, where the expectations over states in equations \ref{eq:model-obj},~\ref{eq:bc-obj},~\ref{eq:trpo-obj} are replaced with supremums and the dynamics model is set to have unit variance.
\begin{proposition}[Policy and model error bound]
    \label{prop:policy_model_error}
    Suppose we apply the idealized BREMEN on a dataset $\mathcal{D}$, and define $\epsilon_\beta,\epsilon_\phi$ in terms of the behavior cloning and dynamics model losses as,
    \begin{align*}
        \epsilon_\beta &:= \sup_{s} \E_{a\sim\mathcal{D}(-|s)} [\left\|a_{t}-\hat{\pi}_{\beta}\left(s_{t}\right)\right\|_{2}^{2} / 2]  - \mathcal{H}(\pi_b(-|s))
        \\
        \epsilon_\phi &:= \sup_{s,a} \E_{s'\sim\mathcal{D}(-|s,a)} \left[\|s' - \hat{f}_\phi(s,a)\|_2^2 / 2 \right] - \mathcal{H}(p(-|s,a)),
    \end{align*}
    where $\mathcal{H}$ denotes the Shannon entropy.
    If one then applies $T$ KL-based trust-region steps of step size $\delta$ (equation~\ref{eq:trpo-obj}) using stochastic dynamics models with mean $\hat{f}_\phi$ and standard deviation 1,
    then
    \begin{equation*}
        \epsilon_\pi \le \sqrt{\frac{1}{2}\epsilon_\beta +\frac{1}{4}\log 2\pi} + T\sqrt{\frac{1}{2}\delta}
        ~;~~~
        \epsilon_m \le \sqrt{\frac{1}{2}\epsilon_\phi + \frac{1}{4}\log 2\pi}.
    \end{equation*}
\end{proposition}
\begin{proof}
See Appendix~\ref{sec:proof_of_propotion}.
\end{proof}

\section{Experiments}
We evaluate BREMEN in both deployment-efficient settings, where the algorithm must learn a policy from scratch via a limited number of deployments, and offline RL, where the algorithm is given only a single static dataset.
We use four standard continuous control benchmarks for offline RL~\cite{kumar2019stabilizing,Wu:2019wl}, namely, Ant, HalfCheetah, Hopper, and Walker2d on the MuJoCo physics simulator~\cite{todorov2012mujoco}.
See Appendix~\ref{sec:experimental_appendix} and~\ref{sec:additional_results} for further details and results.

\subsection{Evaluating Deployment Efficiency}
\label{sec:exp_dep_eff}
We compare BREMEN to ME-TRPO, SAC, BCQ, and BRAC applied to limited deployment settings. 
To adapt offline methods~(BCQ, BRAC) to this setting, we simply apply them in a recursive fashion;\footnote{Recursive BCQ and BRAC also do behavioral cloning-based policy initialization after each deployment.} at each deployment iteration, we collect a batch of data with the most recent policy and then run the offline update with this dataset.
As for SAC, we simply change the replay buffer to update only at specific deployment intervals.
For the sake of comparison, we align the number of deployments and the amount of data collection at each deployment~(either 100,000 or 200,000) for all methods.
Figure~\ref{fig:deployment_result} shows the results with 200,000~(top) and 100,000~(bottom) batched transitions per deployment.
Regardless of the environments and the batch size per update, BREMEN achieves remarkable performance while existing online and offline RL methods struggle to make any progress in the limited deployment settings. 
As a point of comparison, we also include results for online SAC and ME-TRPO without limits on the number of deployments but using the same number of transitions.

\begin{figure}[t]
    \centering
    \includegraphics[width=\linewidth]{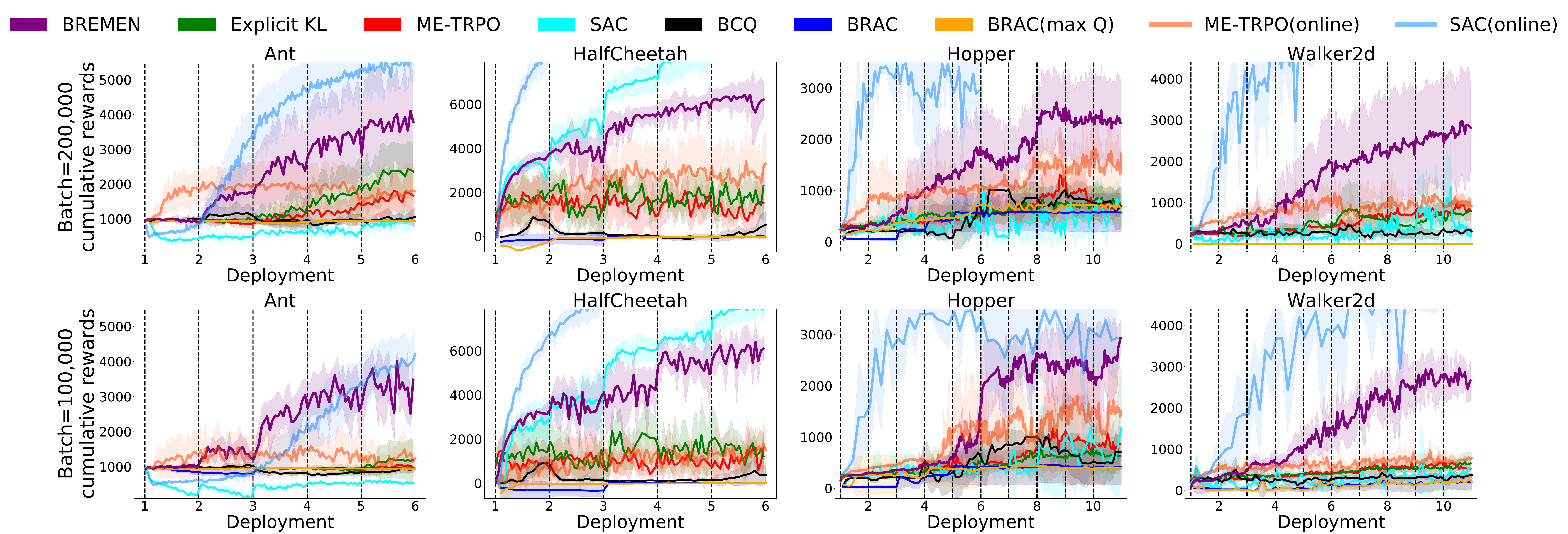}
    \caption{
    Evaluation of BREMEN with the existing methods~(ME-TRPO, SAC, BCQ, BRAC) under deployment constraints (to 5-10 deployments with batch sizes of 200k and 100k).
    The average cumulative rewards and their standard deviations with 5 random seeds are shown.
    Vertical dotted lines represent where each policy deployment and data collection happen.
    BREMEN is able to learn successful policies with only 5-10 deployments, while the state-of-the-art off-policy (SAC), model-based (ME-TRPO), and recursively-applied offline RL algorithms (BCQ, BRAC) often struggle to make any progress.
    For completeness, we show ME-TRPO(online) and SAC(online) which are their original optimal learning curves without deployment constraints, plotted with respect to samples normalized by the batch size. While SAC(online) substantially outperforms BREMEN in sample efficiency, it uses 1 deployment per sample, leading to 100k-500k deployments required for learning. Interestingly, BREMEN achieves even better performance than the original ME-TRPO(online), suggesting the effectiveness of implicit behavior regularization. For SAC and ME-TRPO under deployment-constrained evaluation, their batch size between policy deployments differs substantially from their standard settings, and therefore we performed extensive hyper-parameter search on the relevant parameters such as the number of policy updates between deployments, as discussed in Appendix~\ref{sec:hyper_deploy_rl}. 
    \label{fig:deployment_result}}
    \vspace{-0.5cm}
\end{figure}

\subsection{Evaluating Offline Learning}
\label{sec:exp_off}
We also evaluate BREMEN on standard offline RL benchmarks following~\citet{Wu:2019wl}. We first train online SAC to a certain cumulative reward threshold, 4,000 in HalfCheetah, 1,000 in Ant, Hopper, and Walker2d, and collect offline datasets. %
We evaluate agents with the offline dataset of one million~(1M) transitions, which is standard for BCQ and BRAC~\cite{Wu:2019wl}. We then evaluate them on much smaller datasets of 50k and 100k transitions, 5$\sim$10 \% of prior works. 

Table~\ref{table:offline_rl} shows that BREMEN can achieve performance competitive with state-of-the-art model-free offline RL algorithms when using the standard dataset size of 1M.
Moreover, BREMEN can also appropriately learn with 10-20 times smaller datasets, where BCQ and BRAC are unable to exceed even BC baseline.
As a result, our recursive BREMEN algorithm is not only deployment-efficient but also sample-efficient, and significantly outperforms the baselines. 

\begin{table}[t]
  \caption{Comparison of BREMEN to the existing offline methods on static datasets. Each cell shows the average cumulative reward and their standard deviation, where the number of samples is 1M, 100K, and 50K, respectively. The maximum steps per episode is 1,000. BRAC applies a primal form of KL value penalty, and BRAC (max Q) means its variant of  sampling multiple actions and taking the maximum according to the learned Q function.}
  \label{table:offline_rl}
  \small
  \centering
    \begin{tabular}{l|c|c|c|c}
    \bhline{1.1pt}
    \multicolumn{5}{c}{\textbf{1,000,000~(1M) transitions}}\\
    \hline
    Method &Ant &HalfCheetah &Hopper &Walker2d \\
    \hline
    Dataset &1191 &4126 &1128 &1376 \\
    BC &1321$\pm$141 & 4281$\pm$12 & 1341$\pm$161 & 1421$\pm$147 \\
    BCQ~\cite{Fujimoto:2018td} & 2021$\pm$31 & 5783$\pm$272 & 1130$\pm$127 & 2153$\pm$753 \\
    BRAC~\cite{Wu:2019wl} & 2072$\pm$285 & 7192$\pm$115 & 1422$\pm$90 & 2239$\pm$1124 \\
    BRAC~(max Q) & 2369$\pm$234 & 7320$\pm$91 & 1916$\pm$343 &\textbf{2409$\pm$1210} \\
    BREMEN~(Ours) & \textbf{3328$\pm$275} & \textbf{8055$\pm$103} & \textbf{2058$\pm$852} & 2346$\pm$230 \\
    ME-TRPO~(offline)~\cite{Kurutach:2018wq} & 1258$\pm$550 & 1804$\pm$924 & 518$\pm$91 & 211$\pm$154 \\
    \bhline{1.1pt}
    \multicolumn{5}{c}{\textbf{100,000~(100K) transitions}}\\
    \hline
    Method &Ant &HalfCheetah &Hopper &Walker2d \\
    \hline
    Dataset &1191 &4066 &1128 &1376 \\
    BC & 1330$\pm$81 & 4266$\pm$21 & 1322$\pm$109 & 1426$\pm$47 \\
    BCQ & 1363$\pm$199 & 3915$\pm$411 & 1129$\pm$238 & \textbf{2187$\pm$196} \\
    BRAC & -157$\pm$383 & 2505$\pm$2501 & 1310$\pm$70 & 2162$\pm$1109 \\
    BRAC~(max Q) & -226$\pm$387 & 2332$\pm$2422 & 1422$\pm$101 & 2164$\pm$1114 \\
    BREMEN~(Ours) & \textbf{1633$\pm$127} & \textbf{6095$\pm$370} & \textbf{2191$\pm$455} & 2132$\pm$301 \\
    ME-TRPO~(offline) & 974$\pm$4 & 2$\pm$434 & 307$\pm$170 & 10$\pm$61 \\
    \bhline{1.1pt}
    \multicolumn{5}{c}{\textbf{50,000~(50K) transitions}}\\
    \hline
    Method &Ant &HalfCheetah &Hopper &Walker2d \\
    \hline
    Dataset &1191 &4138 &1128 &1376 \\
    BC & 1270$\pm$65 & 4230$\pm$49 & 1249$\pm$61 & 1420$\pm$194 \\
    BCQ & 1329$\pm$95 & 1319$\pm$626 & 1178$\pm$235 & 1841$\pm$439 \\
    BRAC &-878$\pm$244 & -597$\pm$73 & 1277$\pm$102 & 976$\pm$1207 \\
    BRAC~(max Q) & -843$\pm$279 & -590$\pm$56 & 1276$\pm$225 & 903$\pm$1137 \\
    BREMEN~(Ours) & \textbf{1347$\pm$283} & \textbf{5823$\pm$146} & \textbf{1632$\pm$796} & \textbf{2280$\pm$647} \\
    ME-TRPO~(offline) & 938$\pm$32 & -73$\pm$95 & 152$\pm$13 & 176$\pm$343 \\
    \bhline{1.1pt}
  \end{tabular}
\end{table}

\subsection{Evaluating Effectiveness of Implicit KL Control}
\label{sec:implicit-kl}
In this section, we present an experiment to better understand the effect of BREMEN's implicit regularization. 
Figure~\ref{fig:eval_deploy_kl} shows the KL divergence of learned policies from the last deployed policy.
We compare BREMEN to variants of BREMEN that use an explicit KL penalty on value instead of BC initialization~(conservative KL trust-region updates are still used).
We find that the explicit KL without behavior initialization variants learn policies that move farther away from the last deployed policy than behavior initialized policies.
This suggests that the implicit behavior regularization employed by BREMEN is more effective as a conservative policy learning protocol.
\begin{figure}[t]
    \centering
    \includegraphics[width=\linewidth]{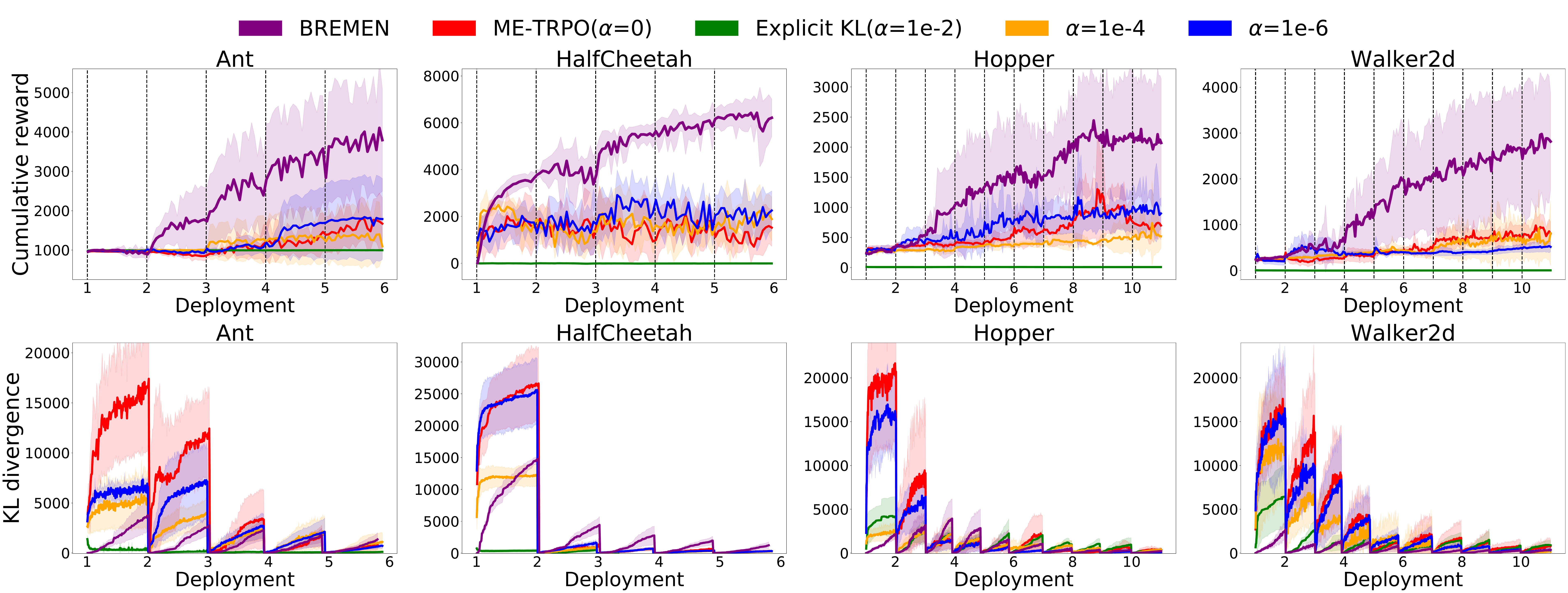}
    \caption{We examine average cumulative rewards~(top) and corresponding KL divergence between the last deployed policy and the target policy~(bottom) with batch size 200K in limited deployment settings. The behavior initialized policy remains close to the last deployed policy during improvement without explicit value penalty $-\alpha \KL(\pi_{\theta} \| \hat{\pi}_{\beta})$. The explicit penalty is controlled by a coefficient $\alpha$.\label{fig:eval_deploy_kl}}
    \vspace{-0.5cm}
\end{figure}

\section{Related Work}

\paragraph{Deployment Efficiency and Offline RL}
Although we are not aware of any previous works which explicitly proposed the concept of deployment efficiency, its necessity in many real-world applications has been generally known.
One may consider previously proposed semi-batch RL algorithms~\cite{ernst2005tree,lange2012batch,singh1994learning,roux2016efficient} or theoretical analysis of switching cost under the tabular PAC-MDP settings~\cite{bai2019provably,guo2015conc} as approaching this issue.
More recently, a related but distinct problem known as offline RL has gained popularity~\cite{levine2020offline,Wu:2019wl,agarwal2019optimistic}.
These offline RL works consider an extreme version of 1 deployment, and typically collect the static batch with a partially trained policy rather than a random policy.
While offline RL has shown promising results for a variety of real-world applications, such as robotics~\cite{IRIS}, dialogue systems~\cite{jaques2019way}, or medical treatments~\cite{gottesman2018evaluating}, these algorithms struggle when learning a policy from scratch or when the dataset is small.
Nevertheless, common themes of many offline RL algorithms -- regularizing the learned policy to the behavior policy~\cite{Fujimoto:2018td,jaques2019way,kumar2019stabilizing,Siegel2020KeepDW,Wu:2019wl} and utilizing ensembles to handle uncertainty~\cite{kumar2019stabilizing,Wu:2019wl} -- served as inspirations for the proposed BREMEN algorithm.
A major difference of BREMEN from prior works is that the target policy is not explicitly forced to stick close to the estimated behavior policy through the policy update.
Rather, BREMEN employs a more implicit regularization by initializing the learned policy with a behavior cloned policy and then applying conservative trust-region updates.
Another major difference is the application of model-based approaches to fully offline settings, which has not been extensively studied in prior works~\cite{levine2020offline}, except the two concurrent works from~\citet{Kidambi2020MOReLM} and \citet{yu2020mopo} that study pessimistic or uncertainty penalized MDPs with guarantees -- closely related to~\citet{liu2019off}.
By contrast, our work shows that a simple technique can already enable model-based offline algorithms to significantly outperform the prior model-free methods, and is, to the best of our knowledge, the first to define and extensively evaluate deployment efficiency with recursive experiments.

\paragraph{Model-Based RL}
There are many types of model-based RL algorithms~\citep{sutton1991dyna,deisenroth2011pilco,heess2015learning}. A simple algorithmic choice is Dyna-style~\cite{sutton1991dyna}, which uses a parameterized dynamics model to estimate the true MDP transition function, stochastically mapping states and actions to next states. The dynamics model can then serve as a simulator of the environment during policy updates.
Dyna-style algorithms often suffer from the distributional shift, also known as model bias, which leads RL agents to exploit regions where the data is insufficient, and significant performance degradation.
A variety of remedies have been proposed to relieve the problem of model bias, such as the use of multiple dynamics models as an ensemble~\cite{chua2018deep,Kurutach:2018wq,janner2019trust}, meta-learning~\cite{clavera2018model}, energy-based model regularizer~\cite{boney2019regularizing}, game-theoretic framework~\cite{rajeswaran2020game}, and explicit reward penalty for unknown state~\cite{Kidambi2020MOReLM,yu2020mopo}.
Notably, we have employed a subset of these remedies -- model ensembles and trust-region updates~\cite{Kurutach:2018wq} -- for BREMEN.
Compared to existing works, our work is notable for using BC initialization in conjunction with trust-region updates to alleviate the distribution shift of the learned policy from the dataset used to train the dynamics model.

\section{Conclusion}
In this work, we introduced \emph{deployment efficiency}, a novel measure for RL performance that counts the number of changes in the data-collection policy during learning.
To enhance deployment efficiency, we proposed Behavior-Regularized Model-ENsemble~(BREMEN), a novel model-based offline algorithm with implicit KL regularization via appropriate policy initialization and trust-region updates. %
BREMEN shows impressive results in limited deployment settings, obtaining successful policies from scratch in only 5-10 deployments, as it can improve policies offline even when the batch size is 10-20 times smaller than prior works.
Not only can this help alleviate costs and risks in real-world applications, but it can also reduce the amount of communication required during distributed learning and could form the basis for communication-efficient large-scale RL in contrast to prior works~\cite{nair2015massively,espeholt2018impala,espeholt2019seed}.
Most critically, we show that under deployment efficiency constraints, most prior algorithms -- model-free or model-based, online or offline -- fail to achieve successful learning. We hope our work can gear the research community to value deployment efficiency as an important criterion for RL algorithms, and to eventually achieve similar sample efficiency and asymptotic performance as the state-of-the-art algorithms like SAC~\citep{Haarnoja:2018uya} while having the deployment efficiency well-suited for safe and practical real-world reinforcement learning. 
\section*{Broader Impact}
Deployment efficiency is a key-concept for real-world applications for RL because excessive policy deployments may be harmful or costly in robotics, health care, dialog agents, or education.
However, in most prior deep RL literatures and benchmarks, this metric is seldom mentioned and sometimes its disregard is exploited for the best sample-efficiency while effectively allowing 100k-1M of deployments.
Our proposed algorithm BREMEN achieves high deployment efficiency of 5-10 deployments for learning the standard OpenAI Gym MuJoCo benchmarks that no other algorithms can match.
While the final performances are sometimes worse than deployment-unconstrained SAC, we hope our definition and benchmark for deployment efficiency can motivate further research by the community.
Other impact questions are not applicable for this paper.

On the other hand, BREMEN still requires a few online deployments, which may still involve some risks.
Fully safe and efficient RL that can learn from scratch remains an open problem.

\section*{Acknowledgments}
We thank Yusuke Iwasawa, Emma Brunskill, Lihong Li, Sergey Levine, and George Tucker for insightful comments and discussion.

\medskip

\small
\bibliographystyle{plainnat}
\bibliography{arxiv}

\clearpage
\normalsize
\section*{Appendix}
\setcounter{section}{0}
\renewcommand{\thesection}{\Alph{section}}

\section{Proof of Proposition~\ref{prop:policy_model_error}}
\label{sec:proof_of_propotion}
We first consider $\epsilon_\pi$.
The behavior cloning objective in its supremum form is,
\begin{eqnarray}
\epsilon_\beta &=& \sup_{s\in\mathcal{D}} \E_{a\sim\mathcal{D}(-|s)} [\left\|a_{t}-\hat{\pi}_{\beta}\left(s_{t}\right)\right\|_{2}^{2} / 2]  - \mathcal{H}(\pi_b(-|s)) \nonumber \\
  &=&
  \sup_{s\in\mathcal{D}} \E_{a\sim\mathcal{D}(-|s)} \left[-\log \pi_{\theta_0}(a|s) \right]  - \mathcal{H}(\pi_b(-|s)) -\frac{1}{2}\log 2\pi 
  \label{eq:bc_lsm} \nonumber\\
  &=& \sup_{s\in\mathcal{D}} D_{KL}(\pi_{b}(-|s)||\pi_{\theta_0}(-|s)) -\frac{1}{2}\log 2\pi. \nonumber
  \label{eq:policy_mle_kl}
\end{eqnarray}
We apply Pinsker's inequality to the true and estimated behavior policy to yield
\begin{equation}
\sup_{s} D_{TV}(\pi_{b}(-|s)||\pi_{\theta_0}(-|s)) \leq \sqrt{\frac{1}{2}\epsilon_\beta +\frac{1}{4}\log 2\pi}. \nonumber
\end{equation}

By the same Pinsker's inequality, we have,
\begin{equation}
\sup_{s} D_{TV}(\pi_{\theta_k}(-|s)||\pi_{\theta_{k+1}}(-|s)) \leq \sqrt{\delta/2}. \nonumber
\end{equation}
Therefore, by triangle inequality, we have
\begin{equation}
    \epsilon_\pi \le \sup_s D_{TV}(\pi_{b}(-|s)||\pi_{\theta_{T}}(-|s)) \le \sqrt{\frac{1}{2}\epsilon_\beta +\frac{1}{4}\log 2\pi} + T\sqrt{\frac{1}{2}\delta}, \nonumber
\end{equation}
as desired.

We perform similarly for $\epsilon_m$. The model dynamics loss is
\begin{eqnarray}
\epsilon_\phi &=& \sup_{s,a} \E_{s'\sim\mathcal{D}(-|s,a)} \left[\|s' - \hat{f}_\phi(s,a)\|_2^2 / 2 \right] - \mathcal{H}(p(-|s,a)) \nonumber\\
&=& \sup_{s,a} \E_{s'\sim\mathcal{D}(-|s,a)} \left[\log p_{\phi}(s'|s,a) \right]  - \mathcal{H}(p(-|s,a)) - \frac{1}{2}\log 2\pi \nonumber\\
&=& \sup_{s,a} D_{KL}(p(-|s,a)||p_{\phi}(-|s,a)) - \frac{1}{2}\log 2\pi. \nonumber
\end{eqnarray}
We apply Pinsker's inequality to the true dynamics and learned model to yield
\begin{equation}
\epsilon_{m} \le \sup_{s,a} D_{TV} (p(-|s,a)||p_{\phi}(-|s,a)) \leq \sqrt{\frac{1}{2}\epsilon_\phi +\frac{1}{4}\log 2\pi}, \nonumber
\end{equation}
as desired.

\section{Details of Experimental Settings}
\label{sec:experimental_appendix}
\subsection{Implementation Details}
For our baseline methods, we use the open-source implementations of SAC, BC, BCQ, and BRAC published in~\citet{Wu:2019wl}.
SAC and BRAC have (300, 300) Q-Network and (200, 200) policy network.
BC has (200, 200) policy network, and BCQ has (300, 300) Q-Network, (300, 300) policy network, and (750, 750) conditional VAE.
As for online ME-TRPO, we utilize the codebase of model-based RL benchmark~\cite{Wang:2019vw}. 
BREMEN and online ME-TRPO use the policy consisting of two hidden layers with 200 units.
The dynamics model also consists of two hidden layers with 1,024 units.
We use Adam~\cite{adam} as the optimizer with the learning rate of 0.001 for the dynamics model, and 0.0005 for behavior cloning in BREMEN.
Especially in BREMEN and online ME-TRPO, we adopt a linear feature value function to stabilize the training.
BREMEN in deployment-efficient settings takes about two or three hours per deployment on an NVIDIA TITAN V.

To leverage neural networks as Dyna-style~\cite{sutton1991dyna} dynamics models, we modify reward and termination function so that they are not dependent on the internal physics engine for calculation, following model-based benchmark codebase~\cite{Wang:2019vw}; see Table~\ref{tab:reward_func}.
Note that the score of baselines~(e.g., BCQ, BRAC) is slightly different from~\citet{Wu:2019wl} due to this modification of the reward function.
We re-run each algorithm in our environments and got appropriate convergence.

The maximum length of one episode is 1,000 steps without any termination in Ant and HalfCheetah; however, termination function is enabled in Hopper and Walker2d.
The batch size of transitions for policy update is 50,000 in BREMEN and ME-TRPO, following~\citet{Kurutach:2018wq}.
The batch size of BC and BRAC is 256, and BCQ is 100, also following~\citet{Wu:2019wl}.

\begin{figure}[ht]
    \centering
    \subfloat[Ant]{\includegraphics[width=0.2\linewidth]{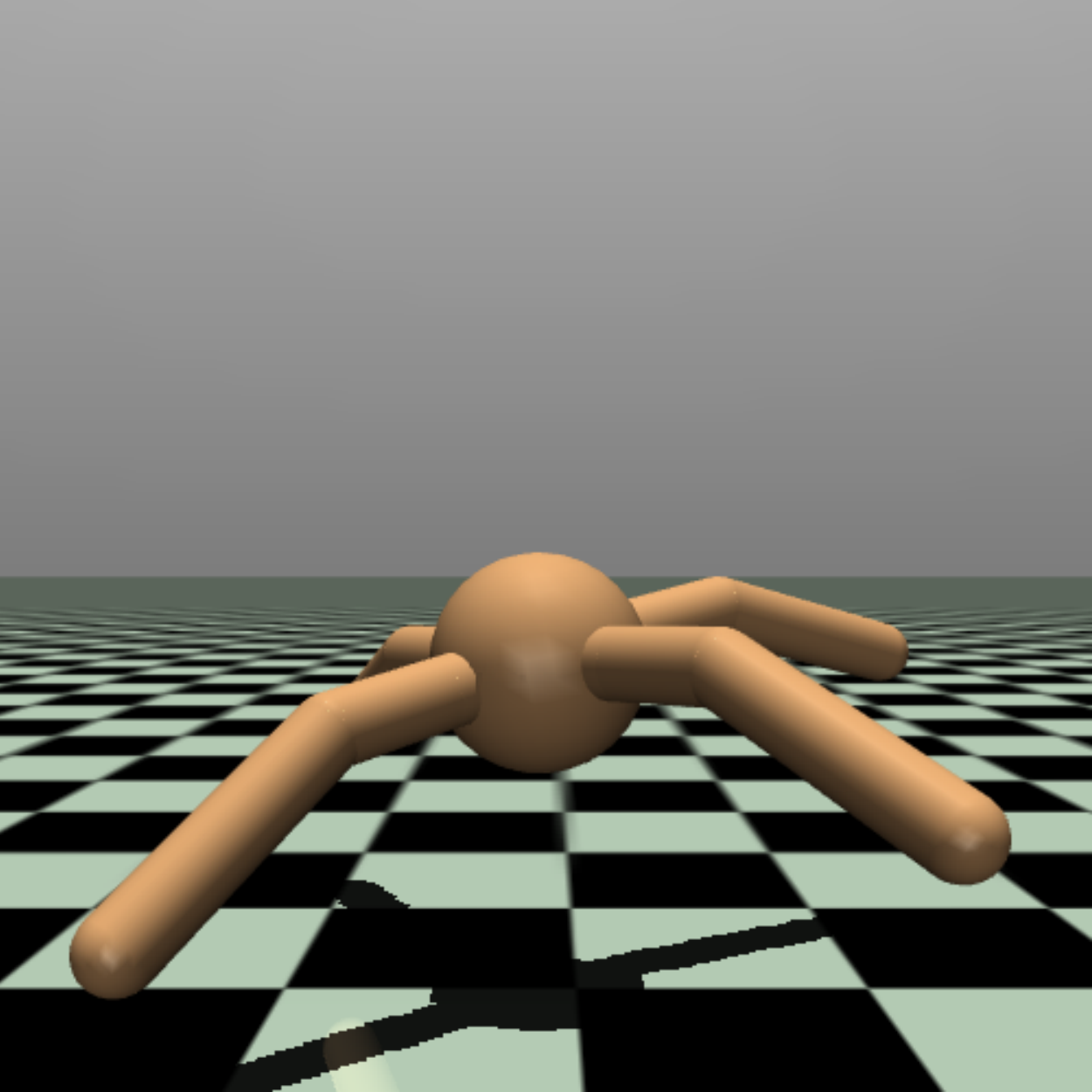}}\,
    \subfloat[HalfCheetah]{\includegraphics[width=0.2\linewidth]{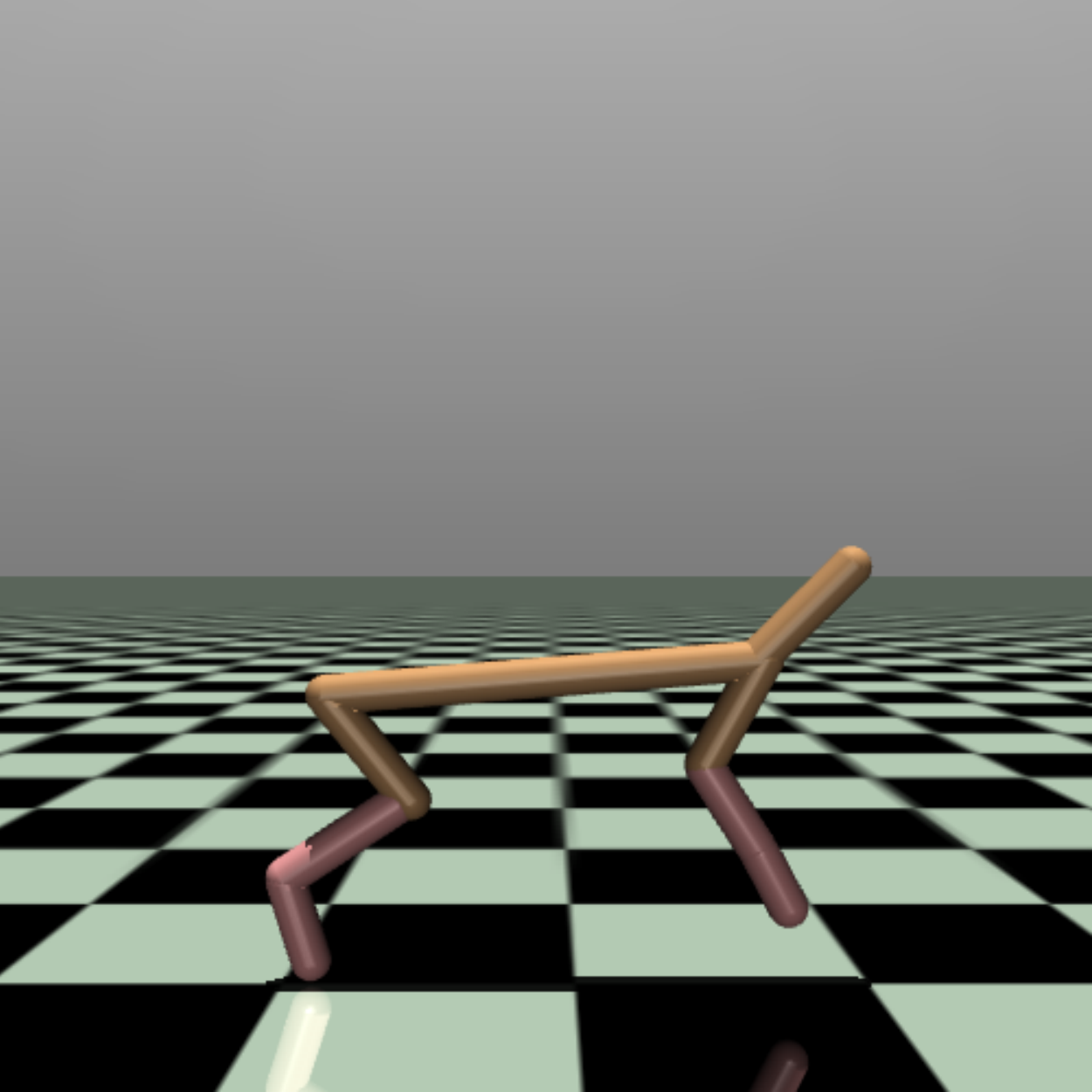}}\,
    \subfloat[Hopper]{\includegraphics[width=0.2\linewidth]{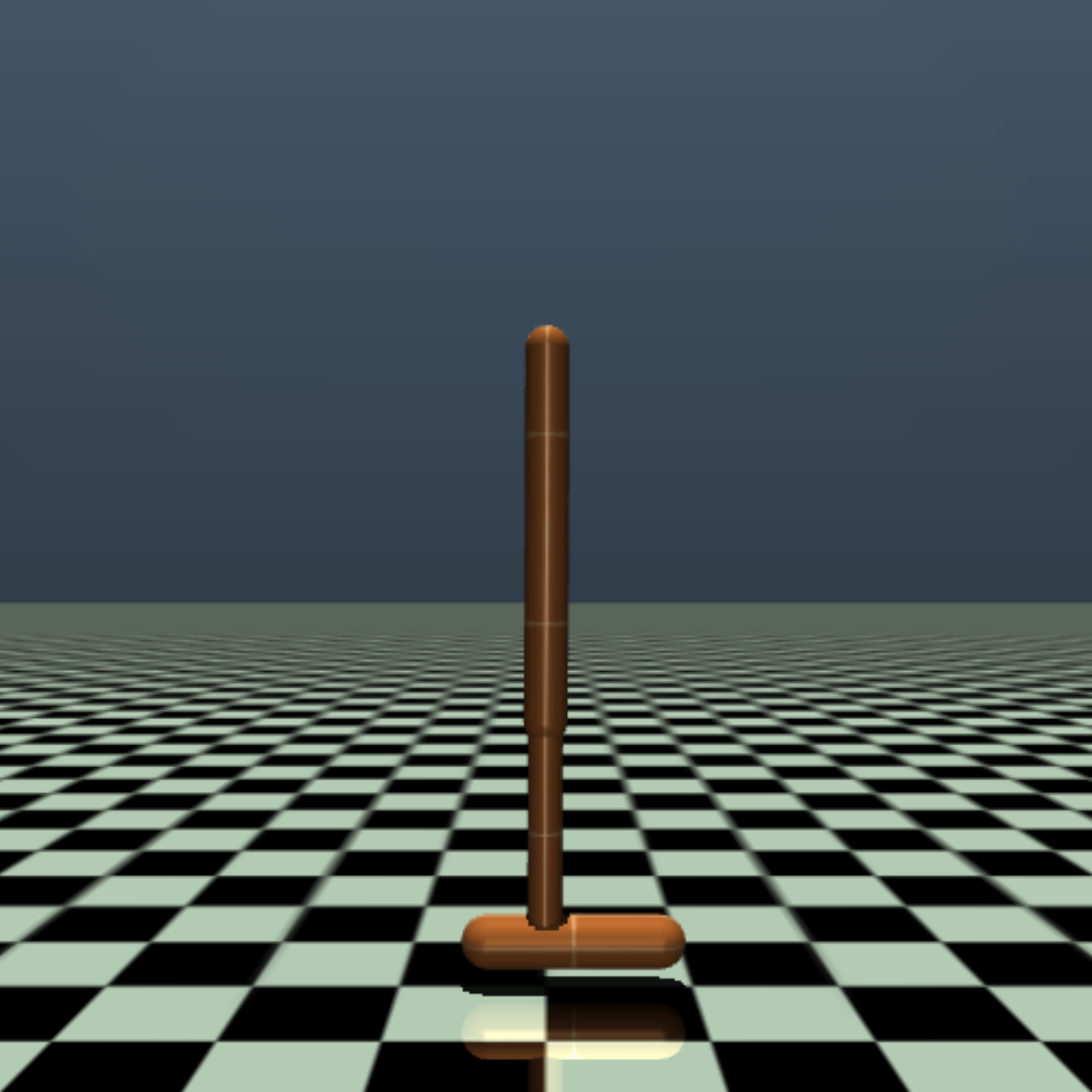}}\,
    \subfloat[Walker2d]{\includegraphics[width=0.2\linewidth]{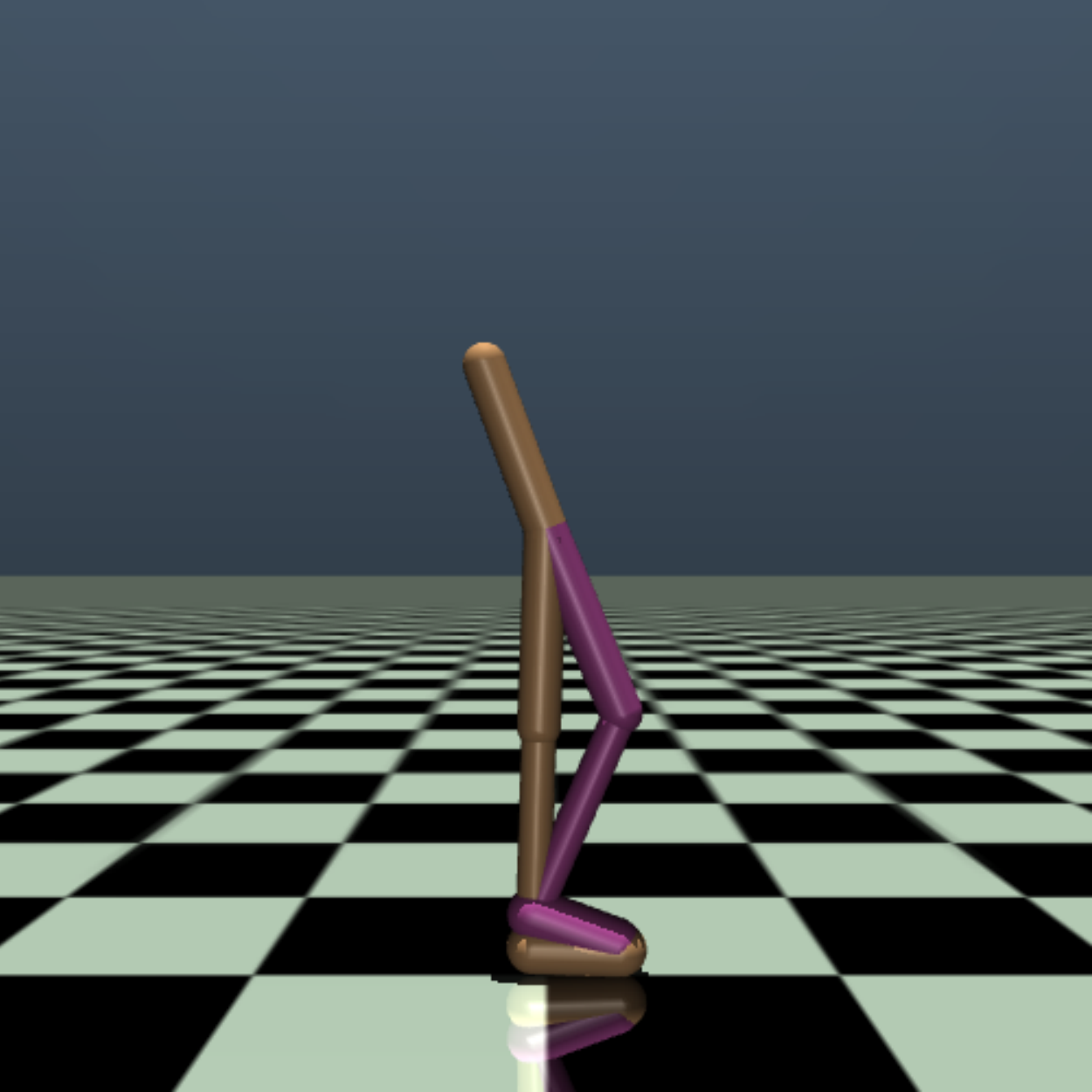}}
    \caption{Four standard MuJoCo benchmark environments used in our experiments.}
    \label{fig:mujoco}
\end{figure}

\begin{table}[ht]
    \centering
    \begin{tabular}{c|c|c}
        \bhline{1.1pt}
        Environment & Reward function & Termination in rollouts \\
        \hline
        Ant &$\dot{x}_t - 0.1\|\bm{a}_t\|_{2}^2 -3.0 \times (z_t - 0.57)^2 + 1$& False\\
        HalfCheetah &$\dot{x}_t - 0.1\|\bm{a}_t\|_{2}^2$ & False\\
        Hopper & $\dot{x}_t - 0.001\|\bm{a}_t\|_{2}^2 + 1$  & True\\
        Walker2d & $\dot{x}_t - 0.001\|\bm{a}_t\|_{2}^2 + 1$ & True\\
        \bhline{1.1pt}
    \end{tabular}
    \caption{Reward function and termination in rollouts in the experiments. We remove all contact information from observation of Ant, basically following~\citet{Wang:2019vw}.}
    \label{tab:reward_func}
\end{table}

\subsection{Hyper Parameters}
In this section, we describe the hyper-parameters in both deployment-efficient RL~(Section~\ref{sec:hyper_deploy_rl}) and offline RL~(Section~\ref{sec:hyper_offline}) settings.
We run all of our experiments with five random seed, and the results are averaged.

\subsubsection{Deployment-Efficient RL}
\label{sec:hyper_deploy_rl}
Table~\ref{tab:deploy_hyper} shows the hyper-parameters of BREMEN.
The rollout length is searched from \{250, 500, 1000\}, and max step size $\delta$ is searched from \{0.001, 0.01, 0.05, 0.1, 1.0\}.
As for the discount factor~$\gamma$ and GAE $\lambda$, we follow~\citet{Wang:2019vw}.

\begin{table}[ht]
    \centering
    \small
    \begin{tabular}{l|c|c|c|c}
    \bhline{1.1pt}
        Parameter & Ant & HalfCheetah & Hopper & Walker2d\\
    \hline
        Iteration per batch & 2,000 & 2,000 & 6,000 & 2,000 \\
        Deployment & 5 & 5 & 10 & 10 \\
        Total iteration & 10,000 & 10,000 & 60,000 & 20,000 \\
        Rollouts length & 250 & 250 & 1,000 & 1,000 \\
        Max step size $\delta$ & 0.05 & 0.1 & 0.05 & 0.05 \\
        Discount factor~$\gamma$ & 0.99 & 0.99 & 0.99 & 0.99\\
        GAE $\lambda$ & 0.97 & 0.95 & 0.95 & 0.95 \\
        Stationary noise $\sigma$ & 0.1 & 0.1 & 0.1 & 0.1 \\
    \bhline{1.1pt}
    \end{tabular}
    \caption{Hyper-parameters of BREMEN in deployment-efficient settings.}
    \label{tab:deploy_hyper}
\end{table}

\paragraph{Number of Iterations for Policy Optimization}
To achieve high deployment efficiency, the number of iterations for policy optimization between deployments is one of the important hyper-parameters for fast convergence.
In the existing methods~(BCQ, BRAC, SAC), we search over three values: \{10,000, 50,000, 100,000\}, and choose 10,000 in BCQ and BRAC, and 100,000 in SAC~(Figure~\ref{fig:iter_search_sac}).
For BREMEN, we also search over three values: \{2,000, 4,000, 6,000\}.
Figure~\ref{fig:iter_search_bremen} shows the results of iteration search, and we choose 2,000 in Ant, HalfCheetah, and Walker2d, and 6,000 in Hopper.

\begin{figure}[ht]
    \centering
    \includegraphics[width=\linewidth]{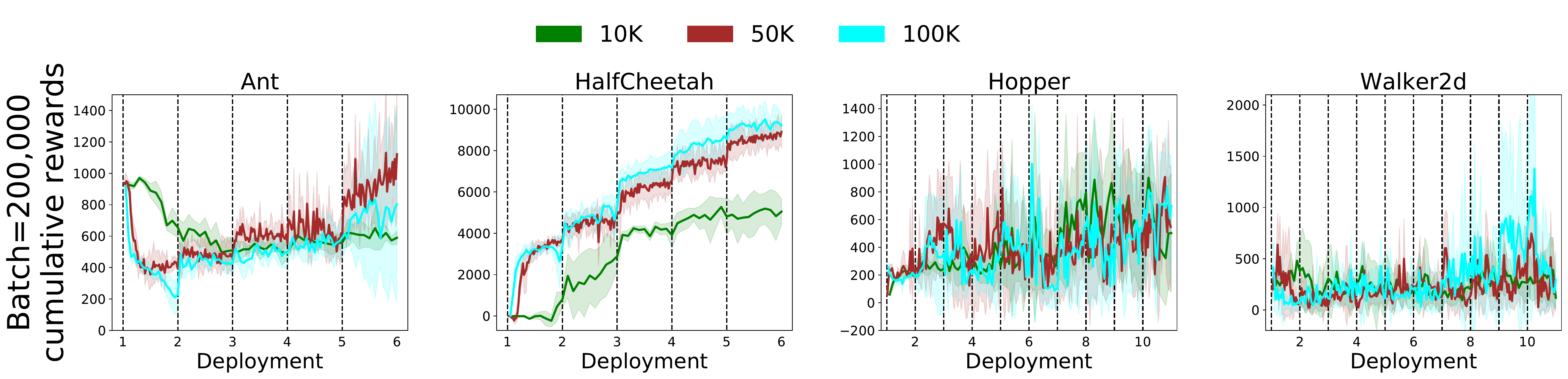}
    \caption{Search on the number of iterations for SAC policy optimization between deployments. The number of transitions per one data-collection is 200K.\label{fig:iter_search_sac}}
\end{figure}

\begin{figure}[ht]
    \centering
    \includegraphics[width=\linewidth]{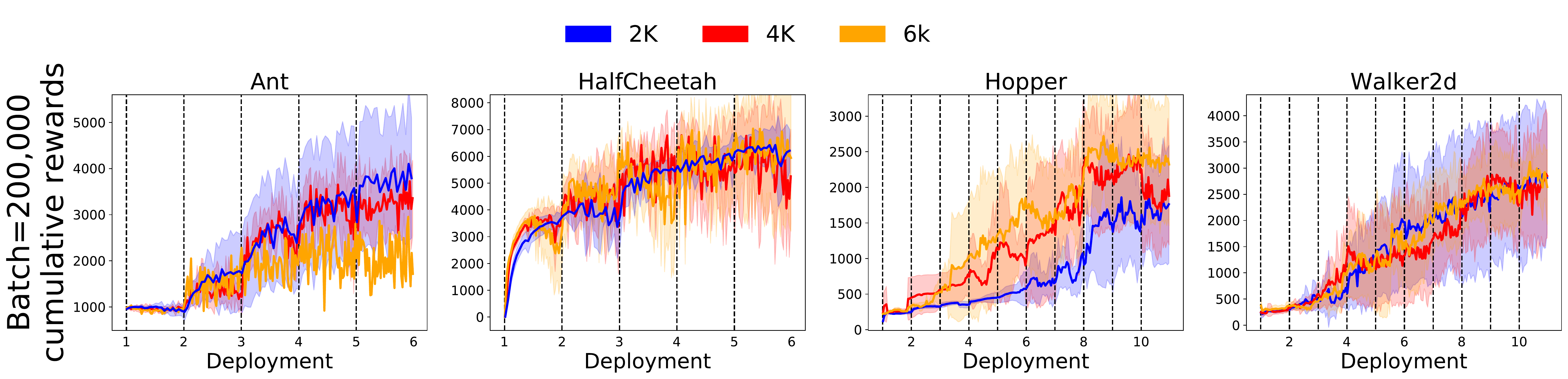}
    \caption{Search on the number of iterations for BREMEN policy optimization between deployments. The number of transitions per one data-collection is 200K.\label{fig:iter_search_bremen}}
\end{figure}

\paragraph{Stationary Noise in BREMEN}
To achieve effective exploration, the stochastic Gaussian policy is a good choice.
We found that adding stationary Gaussian noise to the policy in the imaginary trajectories and data collection led to the notable improvement.
Stationary Gaussian policy is written as,
\begin{equation}
    \label{eq:stationary_policy}
    a_t = \tanh(\mu_{\theta}(s_t)) + \epsilon, \qquad \epsilon \sim \mathcal{N}(0,\sigma^2). \nonumber
\end{equation}
Another choice is a learned Gaussian policy, which parameterizes not only $\mu_{\theta}$ but also $\sigma_{\theta}$.
Learned gaussian policy is also written as,
\begin{equation}
    \label{eq:learned_policy}
    a_t = \tanh(\mu_{\theta}(s_t)) + \sigma_{\theta}(s_t) \odot \epsilon, \qquad \epsilon \sim \mathcal{N}(0,\sigma^2). \nonumber
\end{equation}
We utilize the zero-mean Gaussian $\mathcal{N}(0,\sigma^2)$, and tune up $\sigma$ in Figure~\ref{fig:search_stationary} with HalfCheetah, comparing stationary and learned strategies.
From this experiment, we found that the stationary noise, the scale of 0.1, consistently performs well, and therefore we used it for all our experiments.

\begin{figure}[ht]
    \centering
    \includegraphics[width=10cm]{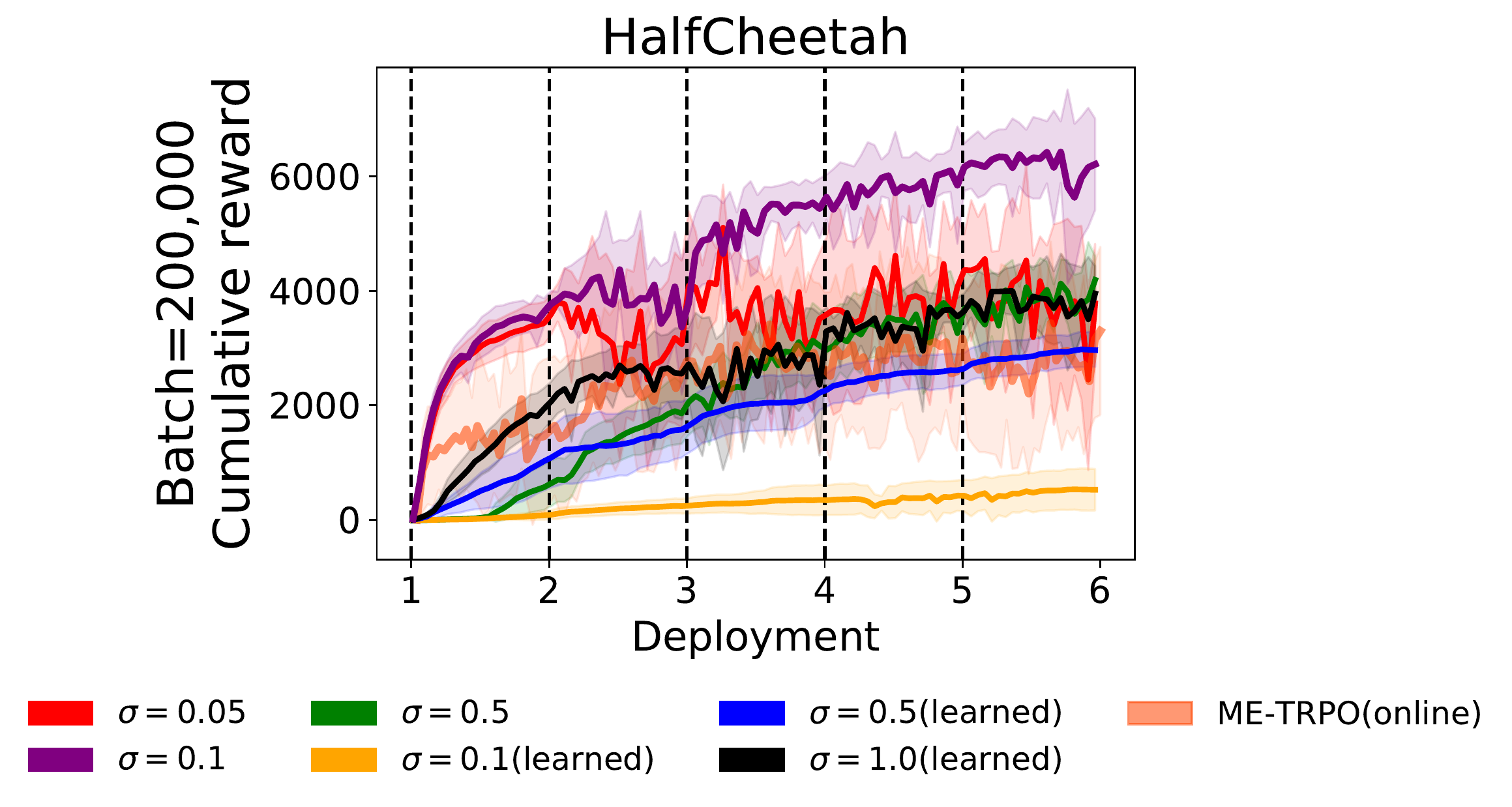}
    \caption{Search on the Gaussian noise parameter $\sigma$ in HalfCheetah. The number of transitions per one data-collection is 200K.\label{fig:search_stationary}}
\end{figure}

\paragraph{Other Hyper-parameters in the Existing Methods}
As for online ME-TRPO, we collect 3,000 steps through online interaction with the environment per 25 iterations and split these transitions into a 2-to-1 ratio of training and validation dataset for learning dynamics models.
In batch size 100,000 settings, we collect 2,000 steps and split with a 1-to-1 ratio.
Totally, we iterate 12,500 times policy optimization, which is equivalent to 500 deployments of the policy.
Note that we carefully tune up the hyper-parameters of online ME-TRPO, and then its performance was improved from \citet{Wang:2019vw}. 

Table~\ref{tab:bcq_hyper} and Table~\ref{tab:brac_hyper} shows the tunable hyper-parameters of BCQ and BRAC, respectively. We refer \citet{Wu:2019wl} to choose these values.
In this work, BRAC applies a primal form of KL value penalty, and BRAC (max Q) means sampling multiple actions and taking the maximum according to the learned Q function.

\begin{table}[ht]
    \centering
    \small
    \begin{tabular}{l|c|c|c|c}
    \bhline{1.1pt}
        Parameter & Ant & HalfCheetah & Hopper & Walker2d\\
    \hline
        Policy learning rate &3e-05 &3e-04 &3e-06 &3e-05 \\
        Perturbation range $\Phi$ &0.15 &0.5 &0.15 &0.15 \\
    \bhline{1.1pt}
    \end{tabular}
    \caption{Hyper-parameters of BCQ.}
    \label{tab:bcq_hyper}
\end{table}

\begin{table}[ht]
    \centering
    \small
    \begin{tabular}{l|c|c|c|c}
    \bhline{1.1pt}
        Parameter & Ant & HalfCheetah & Hopper & Walker2d\\
    \hline
        Policy learning rate &1e-4 &1e-3 &3e-5 &1e-5 \\
        Divergence penalty $\alpha$ &0.3 &0.1 &0.3 &0.3 \\
    \bhline{1.1pt}
    \end{tabular}
    \caption{Hyper-parameters of BRAC.}
    \label{tab:brac_hyper}
\end{table}

\subsubsection{Offline RL}
\label{sec:hyper_offline}
In the offline experiments, we apply the same hyper-parameters as in the deployment-efficient settings described above, except for the iteration per batch.
Algorithm~\ref{alg:bremen} is pseudocode for BREMEN in offline RL settings where policies are updated only with one fixed batch dataset.
The number of iteration $T$ is set to 6,250 in BREMEN, and 500,000 in BC, BCQ, and BRAC.

\begin{algorithm}[ht]
    \small
    \caption{BREMEN for Offline RL}
    \label{alg:bremen}    
    \renewcommand{\algorithmicrequire}{\textbf{Input:}}
    \begin{algorithmic}[1]          
        \REQUIRE Offline dataset $\mathcal{D}=\{s_t,a_t,r_t,s_{t+1}\}$, Initial parameters $\phi = \{\phi_1,\cdots,\phi_K\}$, $\beta$, Number of policy optimization $T$.
        \STATE Train $K$ dynamics models $\hat{f}_{\phi}$ using $\mathcal{D}$  via Eq.~\ref{eq:model-obj}.
        \STATE Train estimated behavior policy $\hat{\pi}_{\beta}$ using $\mathcal{D}$ by behavior cloning via Eq.~\ref{eq:bc-obj}.
        \STATE Initialize target policy $\pi_{\theta_0}=\mathrm{Normal}(\hat{\pi}_{\beta}, 1)$.
        \FOR{policy optimization $k=1,\cdots,T$}
            \STATE Generate imaginary rollout.
            \STATE Optimize target policy $\pi_{\theta}$ satisfying Eq.~\ref{eq:trpo-obj} with the rollout.
        \ENDFOR
    \end{algorithmic}
\end{algorithm}

\section{Additional Experimental Results}
\label{sec:additional_results}
\subsection{Performance on the Dataset with Different Noise}
\label{sec:noise_offline}
Following~\citet{Wu:2019wl} and \citet{Kidambi2020MOReLM}, we additionally compare BREMEN in offline settings to the other baselines~(BC, BCQ, BRAC) with five datasets of different exploration noise.
Each dataset has also one million transitions.
\begin{itemize}
  \item \textbf{eps1}: 40 \% of the dataset is collected by data-collection policy~(partially trained SAC policy) $\pi_b$, 40 \% of the dataset is collected by epsilon greedy policy with $\epsilon=0.1$ to take a random action, and 20 \% of dataset is collected by an uniformly random policy.
  \item \textbf{eps3}: Same as eps1, 40 \% of the dataset is collected by $\pi_b$, 40 \% is collected by epsilon greedy policy with $\epsilon=0.3$, and 20 \% is collected by an uniformly random policy.
  \item \textbf{gaussian1}: 40 \% of the dataset is collected by data-collection policy $\pi_b$, 40 \% is collected by the policy with adding zero-mean Gaussian noise $\mathcal{N}(0,0.1^2)$ to each action sampled from $\pi_b$, and 20 \% is collected by an uniformly random policy.
  \item \textbf{gaussian3}: 40 \% of the dataset is collected by data-collection policy $\pi_b$, 40 \% is collected by the policy with zero-mean Gaussian noise $\mathcal{N}(0,0.3^2)$, and 20 \% is collected by an uniformly random policy.
  \item \textbf{random}: All of the dataset is collected by an uniformly random policy.
\end{itemize}

Table~\ref{table:offline_rl_noise} shows that BREMEN can also achieve performance competitive with state-of-the-art model-free offline RL algorithm even with noisy datasets.
The training curves of each experiment are shown in Section~\ref{sec:train_curve}.

\begin{table}[t]
  \caption{Comparison of BREMEN to the existing offline methods in offline settings, namely, BC, BCQ~\cite{Fujimoto:2018td}, and BRAC~\cite{Wu:2019wl}. Each cell shows the average cumulative reward and their standard deviation with 5 seeds. The maximum steps per episode is 1,000.
  Five different types of exploration noise are introduced during the data collection, eps1, eps3, gaussian1, gaussian3, and random. BRAC applies a primal form of KL value penalty, and BRAC (max Q) means sampling multiple actions and taking the maximum according to the learned Q function.}
  \label{table:offline_rl_noise}
  \small
  \centering
    \begin{tabular}{l|c|c|c|c}
    \bhline{1.1pt}
    \multicolumn{5}{c}{\textbf{Noise: eps1, 1,000,000~(1M) transitions}}\\
    \hline
    Method &Ant &HalfCheetah &Hopper &Walker2d \\
    \hline
    Dataset &1077 &2936 &791 &815 \\
    BC &1381$\pm$71 & 3788$\pm$740 & 266$\pm$486 & 1185$\pm$155 \\
    BCQ & 1937$\pm$116 & 6046$\pm$276 & 800$\pm$659 & 479$\pm$537 \\
    BRAC & 2693$\pm$155 & 7003$\pm$118 & 1243$\pm$162 & 3204$\pm$103 \\
    BRAC~(max Q) & 2907$\pm$98 & 7070$\pm$81 & 1488$\pm$386 & \textbf{3330$\pm$147} \\
    BREMEN~(Ours) & \textbf{3519$\pm$129} & \textbf{7585$\pm$425} & \textbf{2818$\pm$76} & 1710$\pm$429 \\
    ME-TRPO~(offline) & 1514$\pm$503 & 1009$\pm$731 & 1301$\pm$654 & 128$\pm$153 \\
    \bhline{1.1pt}
    \multicolumn{5}{c}{\textbf{Noise: eps3, 1,000,000~(1M) transitions}}\\
    \hline
    Method &Ant &HalfCheetah &Hopper &Walker2d \\
    \hline
    Dataset & 936 & 2408 &662 &648 \\
    BC & 1364$\pm$121 & 2877$\pm$797 & 519$\pm$532 & 1066$\pm$176 \\
    BCQ & 1938$\pm$21 & 5739$\pm$188 & 1170$\pm$446 & 1018$\pm$1231 \\
    BRAC & 2718$\pm$90 & 6434$\pm$147 & 1224$\pm$71 & 2921$\pm$101 \\
    BRAC~(max Q) & 2913$\pm$87 & 6672$\pm$136 & 2103$\pm$746 &\textbf{3079$\pm$110} \\
    BREMEN~(Ours) & \textbf{3409$\pm$218} & \textbf{7632$\pm$104} & \textbf{2803$\pm$65} & 1586$\pm$139 \\
    ME-TRPO~(offline) & 1843$\pm$674 & 5504$\pm$67 & 1308$\pm$756 & 354$\pm$329 \\
    \bhline{1.1pt}
    \multicolumn{5}{c}{\textbf{Noise: gaussian1, 1,000,000~(1M) transitions}}\\
    \hline
    Method &Ant &HalfCheetah &Hopper &Walker2d \\
    \hline
    Dataset &1072 &3150 &882 &1070 \\
    BC &1279$\pm$80 & 4142$\pm$189 & 31$\pm$16 & 1137$\pm$477 \\
    BCQ & 1958$\pm$76 & 5854$\pm$498 & 475$\pm$416 & 608$\pm$416 \\
    BRAC & 2905$\pm$81 & 7026$\pm$168 & 1456$\pm$161 & 3030$\pm$103 \\
    BRAC~(max Q) & 2910$\pm$157 & 7026$\pm$168 & 1575$\pm$89 &\textbf{3242$\pm$97} \\
    BREMEN~(Ours) & \textbf{2912$\pm$165} & \textbf{7928$\pm$313} & \textbf{1999$\pm$617} & 1402$\pm$290 \\
    ME-TRPO~(offline) & 1275$\pm$656 & 1275$\pm$656 & 909$\pm$631 & 171$\pm$119 \\
    \bhline{1.1pt}
    \multicolumn{5}{c}{\textbf{Noise: gaussian3, 1,000,000~(1M) transitions}}\\
    \hline
    Method &Ant &HalfCheetah &Hopper &Walker2d \\
    \hline
    Dataset &1058 &2872 &781 &981 \\
    BC &1300$\pm$34 & 4190$\pm$69 & 611$\pm$467 & 1217$\pm$361 \\
    BCQ &1982$\pm$97 & 5781$\pm$543 & 1137$\pm$582 & 258$\pm$286 \\
    BRAC &3084$\pm$180 & 3933$\pm$2740 & 1432$\pm$499 & 3253$\pm$118 \\
    BRAC~(max Q) &2916$\pm$99 & 3997$\pm$2761 & 1417$\pm$267 & \textbf{3372$\pm$153} \\
    BREMEN~(Ours) &\textbf{3432$\pm$185} & \textbf{8124$\pm$145} & \textbf{1867$\pm$354} & 2299$\pm$474 \\
    ME-TRPO~(offline) &1237$\pm$310 & 2141$\pm$872 & 973$\pm$243 & 219$\pm$145 \\
    \bhline{1.1pt}
    \multicolumn{5}{c}{\textbf{Noise: random, 1,000,000~(1M) transitions}}\\
    \hline
    Method &Ant &HalfCheetah &Hopper &Walker2d \\
    \hline
    Dataset &470 &-285 &34 &2 \\
    BC & 989$\pm$10 & -2$\pm$1 & 106$\pm$62 & 108$\pm$110 \\
    BCQ & 1222$\pm$114 & 2887$\pm$242 & 206$\pm$7 & 228$\pm$12 \\
    BRAC & 1057$\pm$92 & 3449$\pm$259 & 227$\pm$30 & 29$\pm$54 \\
    BRAC~(max Q) & 683$\pm$57 & 3418$\pm$171 & 224$\pm$37 & 26$\pm$50 \\
    BREMEN~(Ours) & 905$\pm$11 & \textbf{3627$\pm$193} & 270$\pm$68 & 254$\pm$6 \\
    ME-TRPO~(offline) & \textbf{2221$\pm$665} & 2701$\pm$120 & \textbf{321$\pm$29} & \textbf{262$\pm$13} \\
    \bhline{1.1pt}
  \end{tabular}
\end{table}

\subsection{Comparison among Different Number of Ensembles}
To deal with the distribution shift during policy optimization, also known as model bias, we introduce the dynamics model ensembles.
We validate the performance of BREMEN with a different number of dynamics models $K$.
Figure~\ref{fig:ensemble_deploy} and Figure~\ref{fig:ensemble_offline} show the performance of BREMEN with the different number of ensembles in deployment-efficient and offline settings.
Ensembles with more dynamics models resulted in better performance due to the mitigation of distributional shift except for $K=10$, and then we choose $K=5$.

\begin{figure}[ht]
    \centering
    \includegraphics[width=\linewidth]{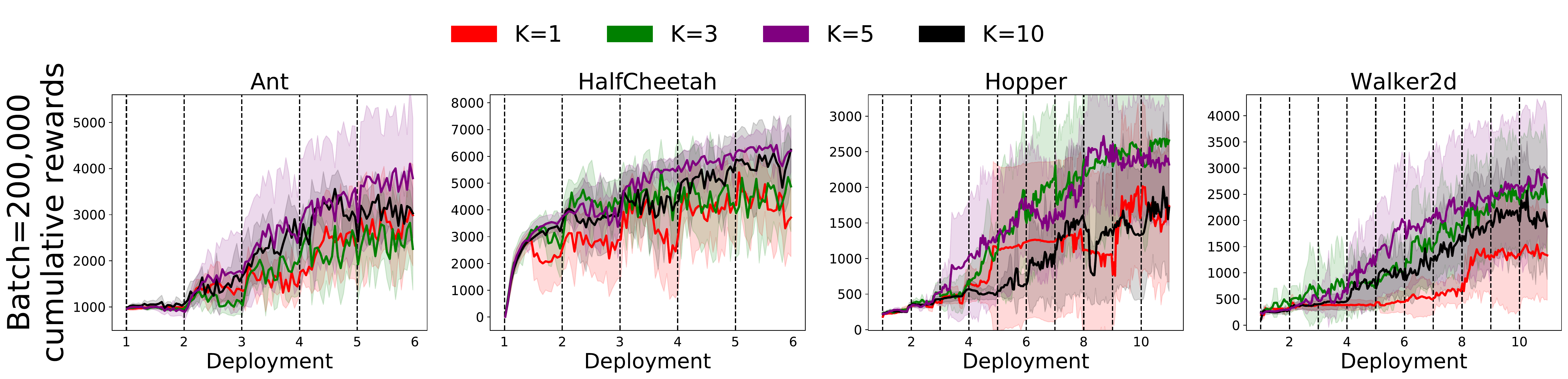}
    \caption{Comparison of the number of dynamics models in deployment-efficient settings.\label{fig:ensemble_deploy}}
\end{figure}

\begin{figure}[ht]
    \centering
    \includegraphics[width=\linewidth]{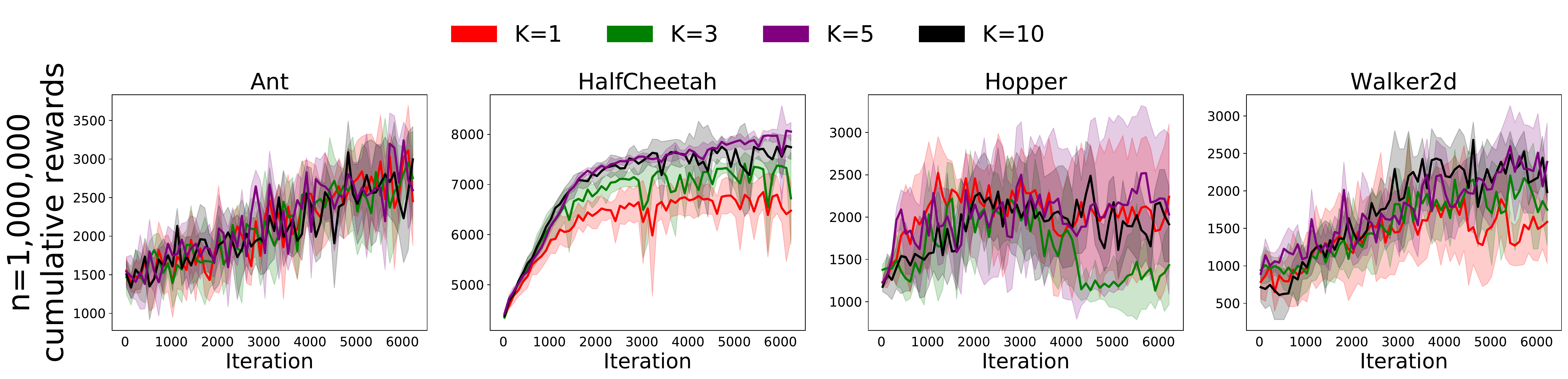}
    \caption{Comparison of the number of dynamics models in offline settings.\label{fig:ensemble_offline}}
\end{figure}

\subsection{Implicit KL Control in Offline Settings}
Similar to Section~\ref{sec:implicit-kl}, we present offline RL experiments to better understand the effect of implicit KL regularization.
In contrast to the implicit KL regularization with Eq.~\ref{eq:trpo-obj}, the optimization of BREMEN with explicit KL value penalty becomes
\begin{align}
    \label{eq:explicit-kl-obj}
    \theta_{k+1} &= \argmax_{\theta} \underset{s, a\sim \pi_{\theta_{k}}, \hat{f}_{\phi_i}}{\operatorname{E}}\left[\frac{\pi_{\theta}(a | s)}{\pi_{\theta_{k}}(a | s)} \left(A^{\pi_{\theta_{k}}}(s, a) - \alpha \KL(\pi_{\theta}(\cdot | s) \| \hat{\pi}_{\beta}(\cdot | s)) \right) \right]\\ \nonumber
    &\text{s.t.} \quad \underset{s \sim \pi_{\theta_k}}{\E}\left[\KL\left(\pi_{\theta}(\cdot | s) \| \pi_{\theta_{k}}(\cdot | s)\right)\right] \leq \delta, \nonumber
\end{align}
where $A^{\pi_{\theta_{k}}}(s, a)$ is the advantage of $\pi_{\theta_k}$ computed using imaginary rollouts with the learned dynamics model and $\delta$ is the maximum step size. Note that BREMEN with explicit KL penalty does not utilize behavior cloning initialization.

We empirically conclude that the explicit constraint $ - \alpha \KL(\pi_{\theta}(\cdot | s) \| \hat{\pi}_{\beta}(\cdot | s))$ is unnecessary and just TRPO update with behavior-initialization as implicit regularization is sufficient in BREMEN algorithm.
Figure~\ref{fig:kl_model_offline} shows the KL divergence between learned policies and the last deployed policies~(top row) and model errors measured by a mean squared error of predicted next state from the true state~(second row).
We find that behavior initialized policy with conservative KL trust-region updates well stuck to the last deployed policy during improvement without explicit KL penalty.
The policy initialized with behavior cloning also tended to suppress the increase of model error, which implies that behavior initialization alleviates the effect of the distribution shift.
In Walker2d, the model error of BREMEN is relatively large, which may relate to the poor performance with noisy datasets in Section~\ref{sec:noise_offline}.

\begin{figure}[ht]
    \centering
    \includegraphics[width=\linewidth]{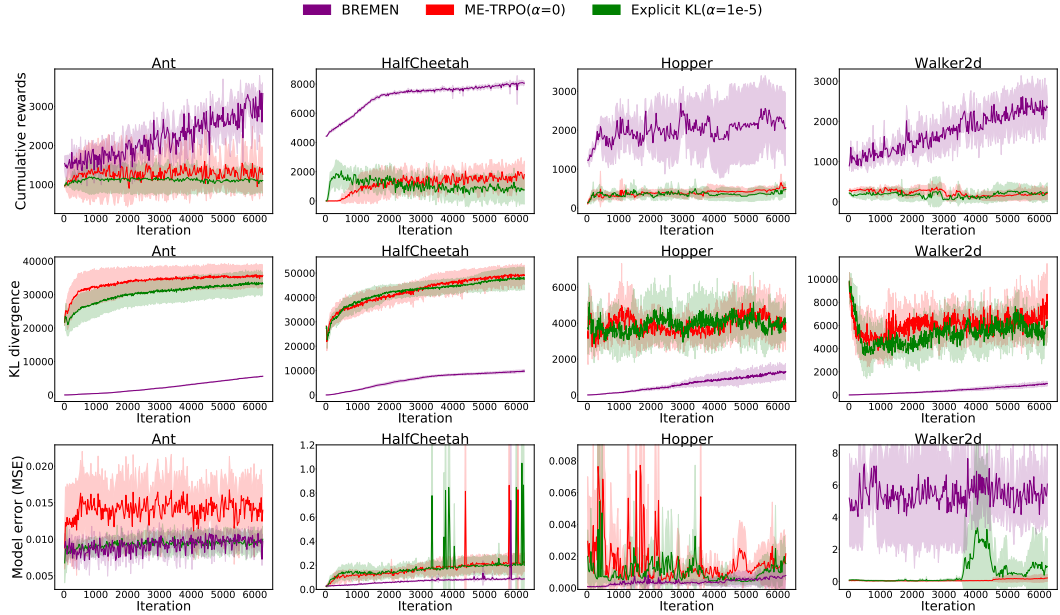}
    \caption{Average cumulative rewards~(top row) and corresponding KL divergence of learned policies from the last deployed policy~(second row) and model errors~(bottom row) in offline settings with 1M dataset~(no noise). Behavior initialized policy~(purple line) tends to suppress the policy and model error during training better than no-initialization~(red line) or explicit KL penalty~(green line).\label{fig:kl_model_offline}}
\end{figure}

\subsection{Training Curves for Offline RL with Different Noises}
\label{sec:train_curve}
In this section, we present training curves of our all experiments in offline settings. Figure~\ref{fig:offline_rl} shows the results in Section~\ref{sec:exp_off}. Figure~\ref{fig:offline_eps1},~\ref{fig:offline_eps3},~\ref{fig:offline_gaussian1},~\ref{fig:offline_gaussian3}, and \ref{fig:offline_random} also show the results in Section~\ref{sec:noise_offline}.

\begin{figure}[ht]
    \centering
    \includegraphics[width=\linewidth]{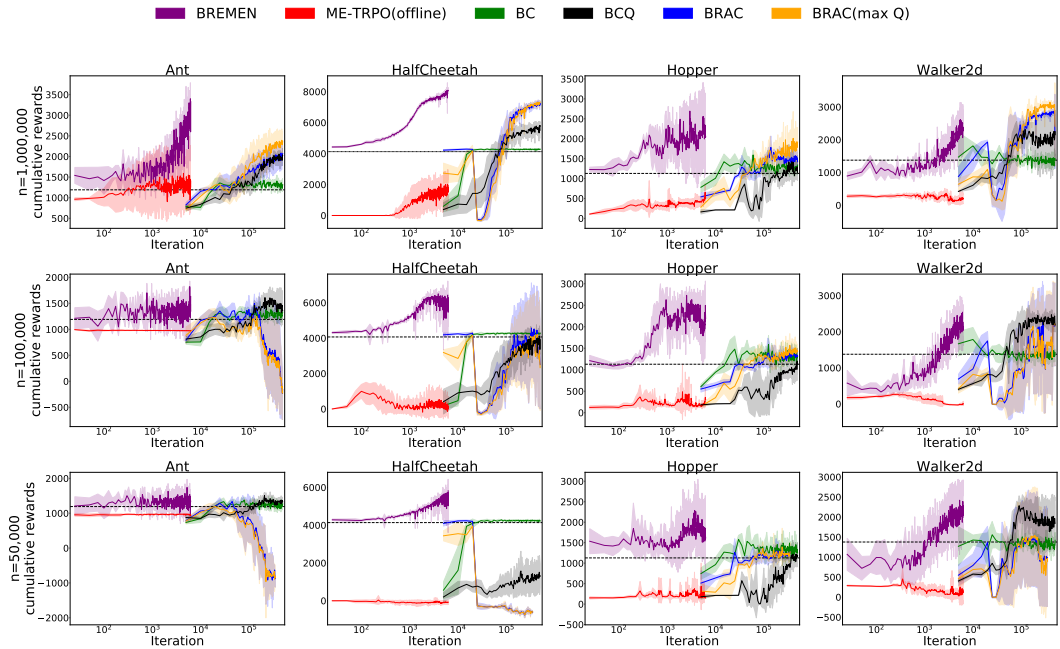}
    \caption{Performance in Offline RL experiments~(Table~\ref{table:offline_rl}). (top row) dataset size is 1M, (second row) 100K, and (bottom row) 50K, respectively. Note that x-axis is the number of iterations with policy optimization in a log-scale.\label{fig:offline_rl}}
\end{figure}

\begin{figure}[ht]
    \centering
    \includegraphics[width=\linewidth]{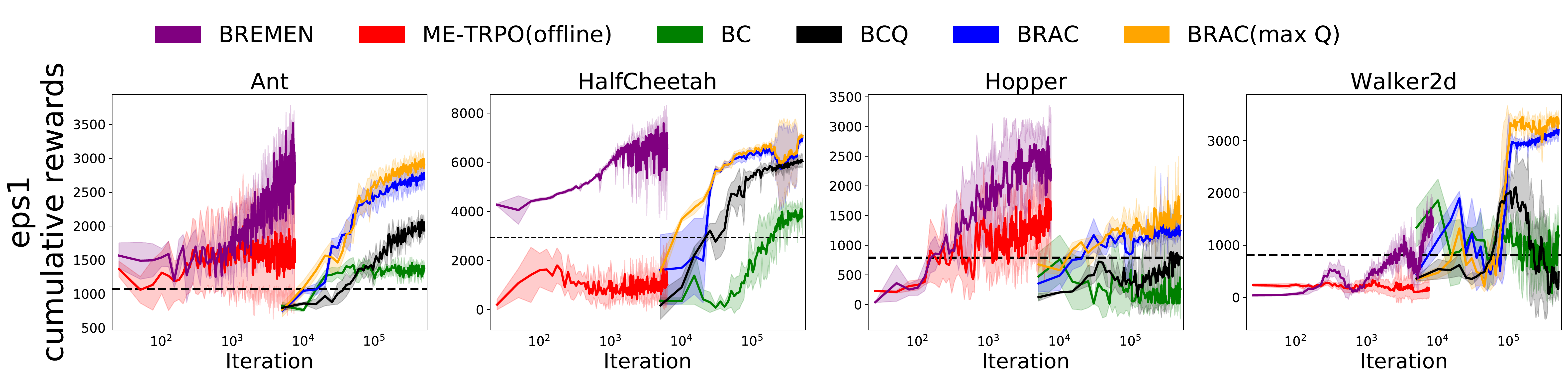}
    \caption{Performance in Offline RL experiments with $\epsilon$-greedy dataset noise $\epsilon = 0.1$. Dataset size is 1M.\label{fig:offline_eps1}}
\end{figure}

\begin{figure}[ht]
    \centering
    \includegraphics[width=\linewidth]{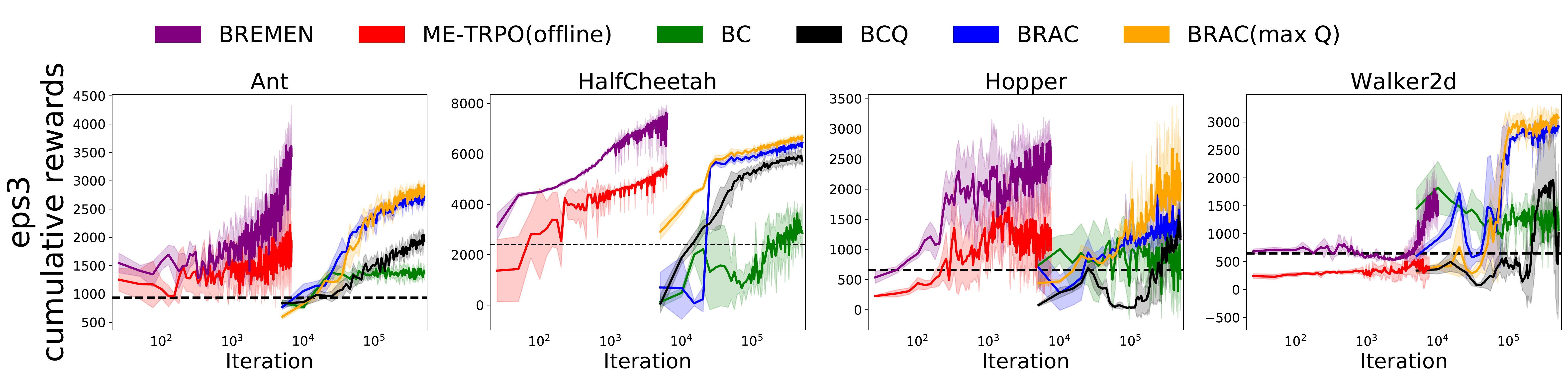}
    \caption{Performance in Offline RL experiments with $\epsilon$-greedy dataset noise $\epsilon = 0.3$. Dataset size is 1M.\label{fig:offline_eps3}}
\end{figure}

\begin{figure}[ht]
    \centering
    \includegraphics[width=\linewidth]{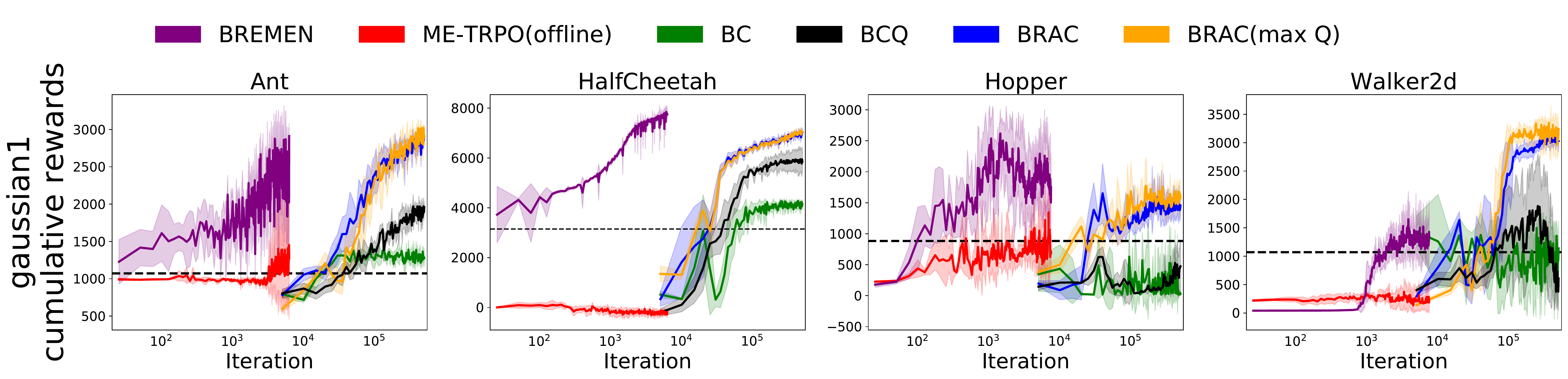}
    \caption{Performance in Offline RL experiments with gaussian dataset noise $\mathcal{N}(0,0.1^2)$. Dataset size is 1M.\label{fig:offline_gaussian1}}
\end{figure}

\begin{figure}[ht]
    \centering
    \includegraphics[width=\linewidth]{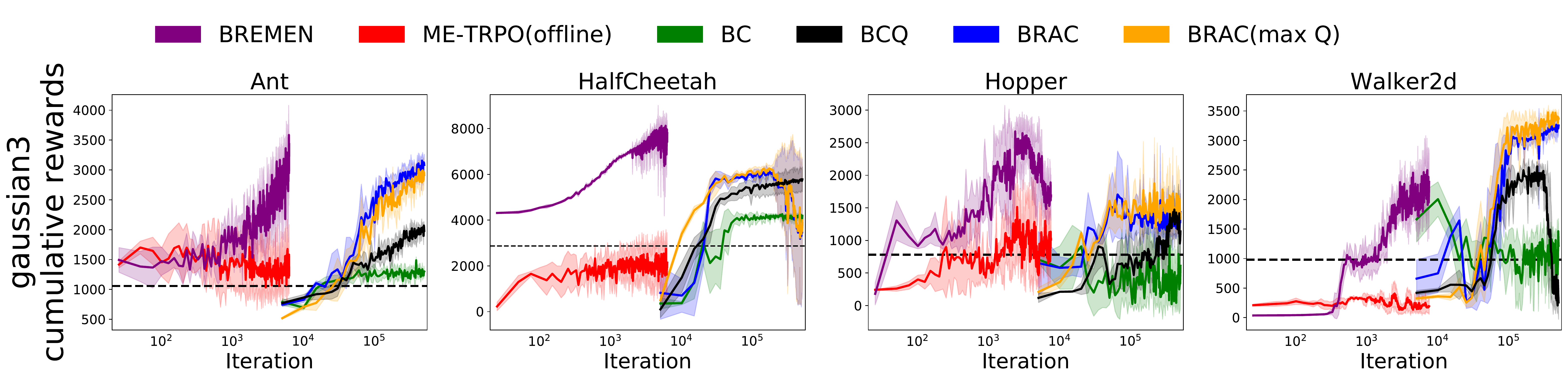}
    \caption{Performance in Offline RL experiments with gaussian dataset noise $\mathcal{N}(0,0.3^2)$. Dataset size is 1M.\label{fig:offline_gaussian3}}
\end{figure}

\begin{figure}[ht]
    \centering
    \includegraphics[width=\linewidth]{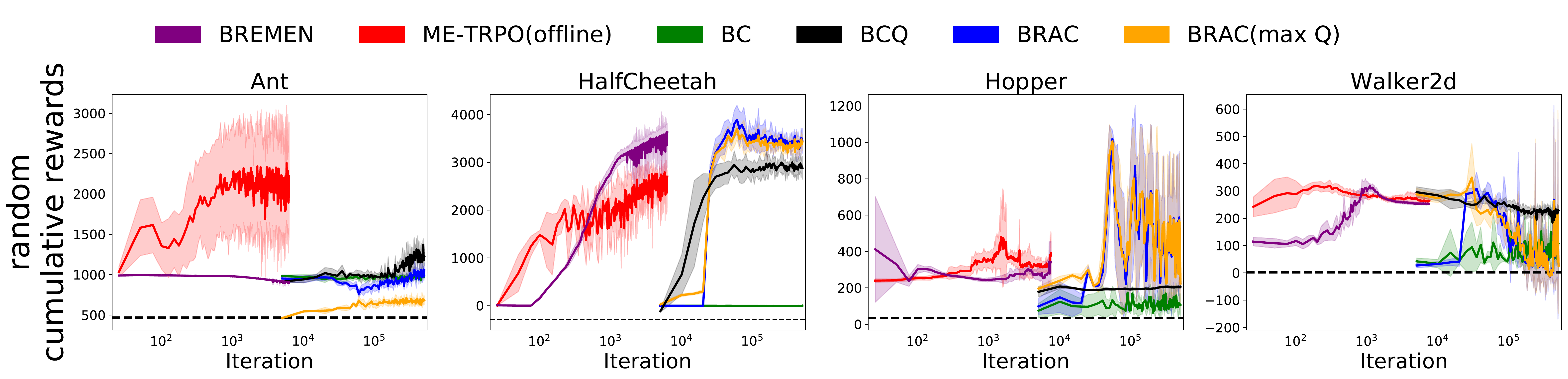}
    \caption{Performance in Offline RL experiments with completely random behaviors. Dataset size is 1M.\label{fig:offline_random}}
\end{figure}

\subsection{Deployment-Efficient RL Experiment with Different Reward Function}
In addition to the main results in Section~\ref{sec:exp_dep_eff} (Figure~\ref{fig:deployment_result}), we also evaluate BREMEN in deployment-efficient setting with different reward function.
We modified HalfCheetah environment into the one similar to cheetah-run task in Deep Mind Control Suite.\footnote{\url{https://github.com/deepmind/dm_control/blob/master/dm_control/suite/cheetah.py}}
The reward function is defined as
\begin{equation}
  r_t =\begin{cases}
    0.1\dot{x}_t & (0 \leq \dot{x}_t \leq 10) \\
    1 & (\dot{x}_t > 10), \nonumber
  \end{cases}
\end{equation}
and the termination is turned off.
Figure~\ref{fig:cheetahrun} shows the performance of BREMEN and existing methods.
BREMEN also shows better deployment efficiency than other existing offline methods and online ME-TRPO, except for SAC, which is the same trend as that of main results.

\begin{figure}[ht]
    \centering
    \includegraphics[width=\linewidth]{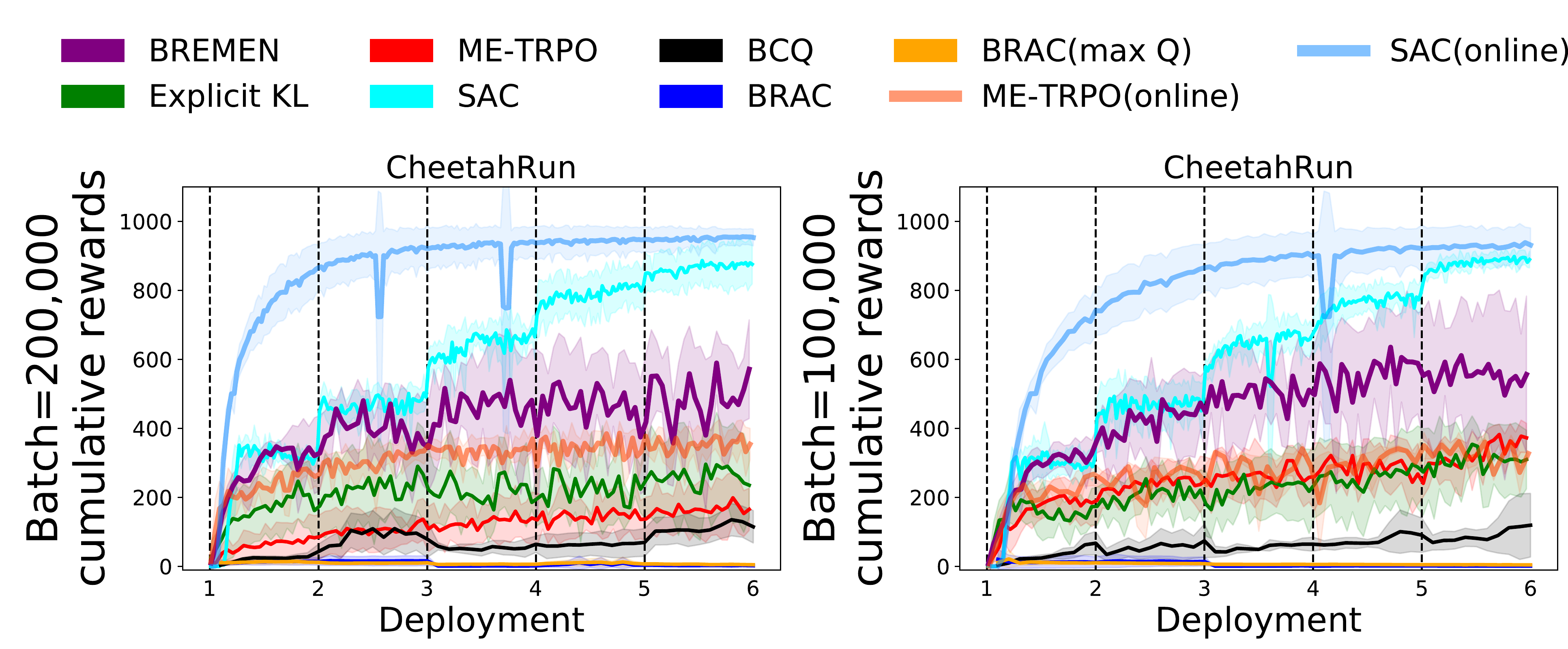}
    \caption{Performance in Deployment-Efficient RL experiments with different reward function of HalfCheetah.\label{fig:cheetahrun}}
\end{figure}

\end{document}

%% file: math_commands.tex
\usepackage{amsmath,amsfonts,bm}

\def\eqref#1{equation~\ref{#1}}
\def\1{\bm{1}}

\DeclareMathAlphabet{\mathsfit}{\encodingdefault}{\sfdefault}{m}{sl}
\SetMathAlphabet{\mathsfit}{bold}{\encodingdefault}{\sfdefault}{bx}{n}

\newcommand{\E}{\mathbb{E}}

\newcommand{\KL}{D_{\mathrm{KL}}}

\DeclareMathOperator*{\argmax}{arg\,max}